\documentclass[twoside,11pt]{article}

\usepackage{blindtext}

\usepackage[abbrvbib, nohyperref, preprint]{jmlr2e}

\usepackage{graphicx}
\usepackage{subcaption}
\usepackage{amsmath}
\usepackage{color}
\usepackage[ruled]{algorithm2e}
\usepackage{float}
\usepackage{mathrsfs} 
\usepackage{booktabs}
\usepackage{amssymb}
\usepackage{lipsum} 
\usepackage{cuted}
\usepackage{flushend}
\usepackage{stfloats}
\usepackage{algorithmicx}
\usepackage{algpseudocode}
\usepackage{multirow}
\usepackage[table,xcdraw]{xcolor}

\usepackage{lastpage}
\jmlrheading{23}{2023}{1-\pageref{LastPage}}{2/23; Revised --/23}{--/23}{21-0000}{Gongpu Chen and Soung Liew}

\ShortHeadings{Intermittently Observable MDP}{G. Chen and S. Liew}
\firstpageno{1}

\begin{document}

\title{Intermittently Observable Markov Decision Processes}

\author{\name Gongpu Chen \email gongpu.chen@imperial.ac.uk \\
	 \addr Department of Electrical and Electronic Engineering\\
	   Imperial College London\\
	   SW7 2AZ, London, UK \\
        \name Soung Chang Liew \email soung@ie.cuhk.edu.hk \\
        \addr Department of Information Engineering\\
       The Chinese University of Hong Kong\\
       Shatin, Hong Kong SAR, China  
}

\editor{My editor}

\maketitle

\begin{abstract}
This paper investigates MDPs with intermittent state information. We consider a scenario where the controller perceives the state information of the process via an unreliable communication channel. The transmissions of state information over the whole time horizon are modeled as a Bernoulli lossy process. Hence, the problem is finding an optimal policy for selecting actions in the presence of state information losses. We first formulate the problem as a belief MDP to establish structural results. The effect of state information losses on the expected total discounted reward is studied systematically. Then, we reformulate the problem as a tree MDP whose state space is organized in a tree structure. Two finite-state approximations to the tree MDP are developed to find near-optimal policies efficiently. Finally, we put forth a nested value iteration algorithm for the finite-state approximations, which is proved to be faster than standard value iteration. Numerical results demonstrate the effectiveness of our methods.
\end{abstract}


\section{Introduction}
\label{sec:introduction}
Markov  decision process (MDP) is a widely adopted model for sequential decision-making in discrete-time stochastic control problems. At each time step, the process is in some state $s$, and a controller is responsible for selecting an action $a$ from the set of available actions of state $s$. Upon the execution of action $a$, the process stochastically transitions to a new state $s'$ at the next time step and generates a reward $r(s,a)$. Solving the MDP consists in finding an optimal policy of selecting actions for the controller to maximize the total reward over a time horizon. It is well-known that the current state of the process is sufficient for computing the optimal action at any time step \citep{puterman1994}. We thus refer to a policy as a mapping from state space to action space.

At every time point of decision-making, the controller needs the current state information to determine the action to be applied. In the classical setting, the current state information is assumed to be always available to the controller. This assumption, however, is not valid in many practical applications. A typical example is a situation where the controller relies on a remote sensor for perceiving the state information of the process; at each time step, the remote sensor observes the current state of the process and transmits the information to the controller via wireless communication. Such a situation is ubiquitous nowadays as wireless sensor network technology is being rapidly developed and deployed \citep{yick2008wireless,kandris2020applications}. {However, despite its many advantages, wireless communication is unreliable in various scenarios. In particular, transmissions over a wireless channel may occasionally fail due to  factors such as channel fading, environmental interference, energy constraints, and failed of channel access \citep{tse2005fundamentals}.} As a result, the state information can not always be delivered successfully from the remote sensor to the controller; hence the controller has to select actions for the MDP in the presence of state information losses. We refer to this problem as intermittently observable MDP (IOMDP) and investigate it in this paper from theoretical and algorithmic perspectives.

Studies on MDPs with imperfect information transmissions can be traced back to \cite{brooks1972markov}, who investigated MDPs with one-step delayed state information. Continuing along this line of research, \cite{Katsi2003} extended the framework and showed that an MDP with delays (whether constant or random delays) can be reduced to an MDP without delays, differing only in the size of state space.  \cite{Goldsmith2012} studied a group of coupled MDPs with delayed state information and provided a bound on the finite history of observations and control needed for the optimal control. In addition, state information delays have also been considered in partially observable MDPs (POMDPs) \citep{kim1987,bander1999} and decentralized control problems \citep{hsu1982decentralized,varaiya1978}.

While delayed state information has been extensively investigated, intermittent state information has not attracted particular attention in MDP studies. Perhaps this is because IOMDPs naturally fall into the category of POMDPs. 
{Specifically, in an IOMDP, each observation either provides complete knowledge of the state (when state information is received) or no information at all (when state information is lost).}
The mainstream method for treating a POMDP is to reformulate it as a fully observable MDP by constructing belief states \citep{astrom1965POMDP,krishnamurthy2016POMDP}. The resulting MDP is thus sometimes called belief MDP \citep{kaelbling1998}.
Under this formulation, it can be proved that the value function of the belief MDP is a piecewise linear and convex function of the belief state \citep{smallwood1973,white1980,araya2010pomdp}.
This nice property has been widely exploited in the design of POMDP algorithms, resulting in the one-pass algorithm \citep{smallwood1973}, the linear support algorithm \citep{Cheng_1988}, incremental pruning \citep{cassandra2013incremental}, the duality-based approach \citep{zhang2010}, and others. In addition, there are also many algorithms that approximate the exact value iteration solution, including the point-based value iteration \citep{pineau2003PBIV,porta2006point}, heuristic search value iteration \citep{smith2012heuristic}, and others \citep{shani2007forward,kurniawati2008sarsop,poupart2011closing}. 

{Given the complexity of solving the belief MDP due to its uncountable belief space, some studies mitigate this challenge through finite-memory approximation and belief quantization \citep{saldi2017asymptotic,zhou2010solving,yu2008near}.}
Most recently, \cite{kara2022near} demonstrated the near optimality of finite-memory policies for POMDPs under certain conditions by quantizing the belief space. A similar study in \cite{Golowich_STOC} focused on the $\gamma$-observable POMDPs. Furthermore, \cite{kara2023convergence} extended analysis to general POMDPs, including those with continuous state space, and derived a bound for the finite memory approximation in terms of a uniform filter stability error. 
These findings inspire the development of finite-memory policies for POMDPs. 
However, these finite-memory policies may be inefficient for IOMDPs. There are cases where the controller may receive no state information over the past  $n$ steps. In such situations, the most recent  $n$ steps of memory provides little information about the current state, as it contains only the past  $n-1$  actions, resulting in a large filter stability error and poor performance. Motivated by this, we develop more efficient algorithms that exploit the unique structure of IOMDPs.


In this paper, we model the transmissions of state information as a Bernoulli process whose parameter (i.e., the state information reception probability) is determined by the communication environment. We first adopt the belief MDP formulation to analyze the effect of state information losses on the optimal value. Then, we reformulate the IOMDP using the sufficient history---a segment of history information that is sufficient for making the optimal decision. The set of sufficient histories can be organized in a tree structure. We thus call the new formulation a tree MDP. Based on the tree MDP formulation, we propose two efficient finite-state approximations for the problem to find near-optimal policies. 

 Our approach differs from existing finite-memory approximations \citep{kara2022near, Golowich_STOC, kara2023convergence} in both model construction and the resulting error bound. First, while prior works estimate the belief state using the most recent $n$ steps of memory, we construct an approximate MDP using the $n$ steps of memory following the last instance of state information reception.  This distinction is crucial because, unlike in standard POMDPs, the most recent $n$ steps may not always be the most informative memory in IOMDPs.
	Second, the upper bounds on value approximation error and performance loss in \citep{kara2022near,Golowich_STOC,kara2023convergence} are derived in terms of an upper bound on the filter stability error, which depends on how quickly the system forgets its initial distribution. In contrast, our method does not rely on filter stability. Instead, we establish a value approximation error bound and an $\epsilon$-optimality condition that depend on the probability of state information reception (see Theorems \ref{thm:TAL-bound} and \ref{thm:eps-optimal}). This highlights IOMDPs as an interesting special case of POMDPs.

The contributions of this paper are as follows:
\begin{itemize}
	\item We show that the optimal value monotonically increases with the state information reception probability and provide a bound for the performance regret caused by state information losses.
	
	\item We reformulate the IOMDP as a tree MDP and propose a truncated approximation to the tree MDP for finding near-optimal policies. A theoretical bound for the approximation error is derived to better understand the method.
	
	\item We put forth the high-order truncated approximation---a modified truncated approximation that could identify the redundant states (those that will never be visited) and omit them when computing policies. It is more efficient and scalable than the original truncated approximation because of the reduction of state space. 
	
	\item We propose a variant of value iteration, called nested value iteration, for the truncated approximation of the tree MDP. We   show that the nested value iteration converges to the optimal value function faster than the standard value iteration. It also applies to the high-order truncated approximation. 
\end{itemize}

The rest of the paper is organized as follows. Section~\ref{sec:Formulation} presents the preliminaries and problem formulation. Section \ref{sec:Structural} establishes fundamental structural results. Section \ref{sec:TA} introduces the tree MDP formulation and proposes the truncated approximation. Section \ref{sec:HOTA} develops the high-order truncated approximation. Section \ref{sec:NVI} proposes the nested value iteration algorithm. Section \ref{sec:exp} demonstrates the experiment results. Finally, Section \ref{sec: conl} concludes this paper.

\section{Problem Statement}  \label{sec:Formulation}
\subsection{Preliminary: MDP}
An MDP is defined by a tuple $(\mathcal{S}, \mathcal{A}, P, r, \beta)$, where $\mathcal{S}$ is the state space, $\mathcal{A}$ is the action space, $P$ is the transition kernel, $r: \mathcal{S}\times\mathcal{A}\to \mathbb{R}$ is the reward function, $\beta\in [0,1)$ is the discount factor. This paper assumes that both $\mathcal{S}$ and $\mathcal{A}$ are finite sets, and that the function $r$ is bounded. For simplicity, we index the states and actions by letting ${\cal S} = \{ 1,2, \cdots ,|{\cal S}|\} $ and ${\cal A} = \{ 1,2, \cdots ,|{\cal A}|\} $. Let $s_t$ and $a_t$ denote the state and action at time $t$. The following expressions will be used interchangeably: 
\begin{align*}
	{P_a}(i,j) = P(j|i,a) = \Pr\left( {{s_{t + 1}} = j|{s_t} = i,{a_t} = a} \right),
\end{align*}
for any $ a \in {\cal A}$ and $ i,j \in {\cal S}$.
We will use ${P_a} = {[{P_a}(i,j)]_{i,j}}$ to denote the transition matrix associated with action $a$. Solving the MDP consists in finding an optimal policy that maximizes the expected total discounted reward
\begin{align}
	V(s) \buildrel \Delta \over = \max E\left[ {\sum\limits_{t = 0}^\infty  {{\beta ^t}r({s_t},{a_t})|{s_0} = s} } \right],\quad s \in {\cal S}.
\end{align}
More generally, given a distribution $\theta$ of the initial state, we can define 
\begin{align}  \label{eq:Jtheta}
	J(\theta ) = \sum\limits_{s \in {\cal S}} {\theta (s)V(s)} ,\quad \theta  \in {\Delta _{\cal S}},
\end{align}
where $\Delta_{\cal S}$ is the ($|\mathcal{S}|-1$)-dim probability simplex defined as
\begin{align*}
	{\Delta _{\cal S}} \triangleq \left\{ {\theta  \in {\mathbb{R}^{|{\cal S}|}}:\sum\limits_{i = 1}^{|{\cal S}|} {\theta (i)}  = 1,\theta (i) \ge 0,\forall i} \right\}.
\end{align*}
At each time $t$, the state information $s_t\in\mathcal{S}$ is transmitted to the controller, based on which the controller computes an action $a_t$  and applies $a_t$  to the process. Upon the execution of $a_t$, the process generates a reward $r(s_t,a_t)$ and transitions to state $s_{t+1}$ at the next time step. In the classical setting, the transmissions of $s_t$ are always timely and reliable. Then the optimal policy for the MDP can be determined by the Bellman equation:
\begin{align*}
	V(s) = \mathop {\max }\limits_{a \in {\cal A}} \left\{ {r(s,a) + \beta \sum\limits_{y \in {\cal S}} {P(y|s,a)V(y)} } \right\},\quad s \in {\cal S}.
\end{align*}

\subsection{Intermittently Observable MDP}
Consider a scenario where the controller relies on a remote sensor for perceiving the state information of the process. Suppose that the controller and the remote sensor are physically separated and communicate via an unreliable wireless channel. Due to environmental interference and channel fading, transmissions over the wireless channel may occasionally fail. As a result, the state information observed by the sensor may not be always delivered successfully to the controller. This paper proposes the IOMDP framework to address the problem of selecting actions for the MDP with intermittent state information.

Formally, define a Bernoulli random variable $\gamma_t$ as an indicator for the transmission of $s_t$. Let $\gamma_t = 1$ if $s_t$ is successfully transmitted to the controller before the time point of determining $a_t$ and $\gamma_t = 0$ otherwise. We assume there is no delayed arrival of state information, i.e., the controller will never receive $s_t$ if $\gamma_t=0$. We make the following assumption regarding the transmissions of state information:

\noindent {\bf Assumption 1:} {\it
	$\{ {\gamma _t}:t = 0,1,2, \cdots \} $ is a sequence of i.i.d. random variables and $\Pr ({\gamma _t} = 1) = \rho  \in (0,1]$ for all $t\ge 0$. }

Quantity $\rho$ is referred to as the state information reception probability (SIRP). An IOMDP with SIRP $\rho$ is defined by the tuple $(\mathcal{S}, \mathcal{A}, \rho, P, r, \beta)$.

To analyze the IOMDP, we reformulate the problem as a belief MDP denoted by ${\cal B}(\rho ) = ({\cal W},{\cal A},\rho ,T,R,\beta )$. Specifically, ${\cal W}$ is the set of belief states, $T$ is the transition kernel of belief states, $R$ is the reward function. The remaining symbols are of the same meaning as before.
The belief MDP is a special kind of POMDP, in which we use a belief state to represent a probability distribution over the state space $\mathcal{S}$. Let $t_i$ denote the time index of the  $i$-th successfully transmitted state information. Due to possible transmission failures, $t_{i+1} - t_i$ may be greater than 1. Suppose there are in total $k$  successful transmissions up to time $t$ (i.e., $k=\sum_{i=0}^{t}\gamma_i$), then $\{s_{t_i}:1\le i \le k \}$ is the set of all the state information received by the controller up to time $t$. The belief state at time $t$, denoted by $\mathbf{b}_t$, is a sufficient statistic for the given history:
\begin{align} \label{eq:belief}
	{{\bf{b}}_t} &\triangleq \Pr \left( {{s_t}|{s_{{t_1}}},{s_{{t_2}}}, \cdots ,{s_{{t_k}}},{a_0},{a_1}, \cdots ,{a_{t - 1}}} \right) \notag \\
	&= \Pr \left( {{s_t}|{s_{{t_k}}},{a_{{t_k}}},{a_{{t_k} + 1}}, \cdots ,{a_{t - 1}}} \right).
\end{align}
The second equality above follows from the Markovian property. It means that the belief state depends on the latest received state information and the following actions. We will refer to the sequence $({s_{{t_k}}},{a_{{t_k}}},{a_{{t_k} + 1}}, \cdots ,{a_{t - 1}})$  as the sufficient history at time $t$. Note that ${{\bf{b}}_t} \in {\Delta _{\cal S}}$ with its $i$-th element, denoted by $\mathbf{b}_t(i)$, being the probability of $s_t=i$ conditioned on the sufficient history. Specially, if $t_k=t$ (i.e., the controller receives $s_t$), then $\mathbf{b}_t$ reduces to a one-hot vector. We will use $\mathbf{e}_i$ to denote the one-hot vector with the $i$-th element being 1 and other elements being 0. If $s_t$ is not delivered successfully, then $\mathbf{b}_t$ is determined by $\mathbf{b}_{t-1}$ and $a_{t-1}$. In particular, the transition probability of the belief MDP is given by
\begin{align} \label{eq:kernel_T}
	T({{\bf{b}}_t}|{{\bf{b}}_{t - 1}},{a_{t - 1}}) \triangleq \Pr \left( {{{\bf{b}}_t}|{{\bf{b}}_{t - 1}},{a_{t - 1}}} \right) 
	= \begin{cases}
		1 - \rho , &{\text{ if }}{{\bf{b}}_t} = P_{{a_{t - 1}}}^\top{{\bf{b}}_{t - 1}}\\
		\rho {[P_{{a_{t - 1}}}^\top{{\bf{b}}_{t - 1}}]_i}, &{\text{ if }}{{\bf{b}}_t} = {{\bf{e}}_i},i \in {\cal S}\\
		0,&{\rm{                   otherwise}}\quad 
	\end{cases} .
\end{align}
where ${[P_{{a_{t - 1}}}^\top{{\bf{b}}_{t - 1}}]_i}$ denotes the $i$-th element of vector $P_{{a_{t - 1}}}^\top{{\bf{b}}_{t - 1}}$. Finally, the reward function of the belief MDP is given by
\begin{align*}
	R({{\bf{b}}_t},{a_t}) \triangleq \sum\limits_{s \in {\cal S}} {{{\bf{b}}_t}(s)r(s,{a_t})} ,\quad {{\bf{b}}_t} \in {\cal W},{a_t} \in {\cal A}.
\end{align*}
The belief MDP can be viewed as a fully observable MDP; hence its optimal policy can be determined by the Bellman equation:
\begin{align*}
	\phi ({\bf{b}},\rho ) = \mathop {\max }\limits_{a \in {\cal A}} \left\{ {R({\bf{b}},a) + \beta \sum\limits_{{\bf{b'}} \in {\cal W}} {T\left( {{\bf{b'}}|{\bf{b}},a} \right)\phi ({\bf{b'}},\rho )} } \right\},
\end{align*}
where ${\bf{b}} \in {\cal W}$, and $\phi ({\bf{b}},\rho )$ denotes the optimal value function of ${\cal B}(\rho)$. In the context that $\rho$ is fixed, we may omit $\rho$ and simply write the value function as $\phi ({\bf{b}} )$. For every $\mathbf{b} \in {\cal W}$ and $a\in {\cal A}$, $R(\mathbf{b},a)$ can be viewed as the expected one-step reward generated by the underlying MDP given that the current state follows a distribution $\mathbf{b}$. Therefore, $\phi ({\bf{b}},\rho )$ is the maximum expected total discounted reward that can be obtained from the underlying MDP with an initial belief state $\mathbf{b}$ and SIRP $\rho$.  Since $r(s,a)$ is a bounded function, $R(\mathbf{b},a)$ is also bounded. Therefore, it is well-known that an optimal policy for the belief MDP always exists \citep{puterman1994}.

In principle, the belief state can be any element in $ {\Delta _{\cal S}}$ (i.e., ${\cal W} = {\Delta _{\cal S}}$). Then the belief MDP has a continuous state space. However, as implied by \eqref{eq:kernel_T}, given any initial belief state $\theta$, the set of possible belief states is countable, i.e.,
\begin{align*}
	{\Omega _\theta } \buildrel \Delta \over = \left\{ {\prod\limits_{k = 0}^L {P_{{a_k}}^\top} {\bf{b}}:{a_k} \in {\cal A},{\bf{b}} \in \left\{ {\theta ,{{\bf{e}}_i}:i \in {\cal S}} \right\},L = 0,1,2, \cdots } \right\}.
\end{align*}
Unless otherwise specified, we will consider ${\cal W} = {\Omega _\theta }$  instead of ${\Delta _{\cal S}}$  to avoid technique issues related to measurability.

\section{Structural Results}  \label{sec:Structural}
This section establishes structural results that are fundamental for understanding IOMDPs. The aim is to identify the relationship between the state information reception probability and the optimal value function.

Let $\Pi$ denote the set of stationary deterministic policies for the belief MDP. It is well-known that $\Pi$ contains at least one optimal policy; hence we will focus on policies in $\Pi$. For any $\pi\in \Pi$, let $\pi(\mathbf{b})$ denote the action taken by policy $\pi$ in belief state $\mathbf{b}$. Denote by ${\phi ^\pi }( \cdot ,\rho )$ the value function of ${\cal B}(\rho)$ under policy $\pi$ and $T^\pi_\rho$ the associated transition matrix. Their relationship can be expressed in vector form as
\begin{align}
	{\phi ^\pi }[\rho ] = {R_\pi } + \beta T_\rho ^\pi {\phi ^\pi }[\rho ],
\end{align}
where ${\phi ^\pi }[\rho ] = {[{\phi ^\pi }({\bf{b}},\rho )]_{{\bf{b}} \in {\cal W}}}$ is the vector form of the value function over the belief state space ${\cal W}$, and ${R_\pi } = {[R({\bf{b}},\pi ({\bf{b}}))]_{{\bf{b}} \in {\cal W}}}$ is the vector of rewards associated with policy $\pi$. 
Likewise, we will denote by $\phi[\rho]$ the vector form of the optimal value function of ${\cal B}(\rho)$.

We first introduce a useful lemma. For any $\rho\in (0,1]$, ${\cal B}(\rho)$ is a POMDP whose reward function is linear with the belief state. The piecewise linearity and convexity of the optimal value function have been extensively exploited in POMDPs of this kind \cite{smallwood1973}. Formally,
\begin{lemma}
	For any $\rho\in (0,1]$,  $\phi ({\bf{b}},\rho )$ is a piecewise linear and convex function of $\mathbf{b} \in \Delta_{\cal S}$.
\end{lemma}

With Lemma 1, we are ready to establish a key result that reveals the relationship between the optimal value function and the SIRP.
\begin{theorem} \label{theo:phi_b}
	For any $\mathbf{b} \in {\cal W}$, $\phi ({\bf{b}},\rho )$ is an increasing function of $\rho\in (0,1]$.
\end{theorem}
\begin{proof}
	See Appendix A.
\end{proof}

Theorem \ref{theo:phi_b} is consistent with our intuition. It shows that the optimal value function of the belief MDP is monotonically increasing with the state information reception probability. In other words, the more state information the controller can receive, the greater its expected total discounted reward from the underlying MDP. 

The following result is useful in robustness and worst-case performance analysis. 
\begin{corollary}
	For any $\rho\in (0,1]$, suppose that $\pi_\rho\in \Pi$ is an optimal policy for ${\cal B}(\rho)$. Then for any $v\in (0,\rho)$ and $u\in (\rho,1]$, ${\phi ^{{\pi _\rho }}}[v] \le {\phi ^{{\pi _\rho }}}[\rho ] \le {\phi ^{{\pi _\rho }}}[u]$. 
\end{corollary}
\begin{proof}
	The second inequality has been established in the proof of Theorem \ref{theo:phi_b}. Since $\rho>v$, it follows from Theorem \ref{theo:phi_b} that ${\phi ^{{\pi _\rho }}}[\rho ] = \phi [\rho ] \ge \phi [v] \ge {\phi ^{{\pi _\rho }}}[v]$.
\end{proof}

Although straightforward, the lemma below is interesting and worth being highlighted. If we apply a value iteration algorithm to solve ${\cal B}(\rho)$ with any $\rho<1$, $\phi(\cdot,1)$ is a good starting point for the value iteration. As we know, a good starting point may considerably speed up the value iteration algorithm. Since the underlying MDP is much easier to be solved than the belief MDP, the following result can be used in practical computations to reduce computation time.
\begin{lemma} \label{lem:rho=1}
	For $\rho=1$,  the optimal value function $\phi(\cdot,1)$ can be derived from $V$. In particular,
	\begin{itemize}
		\item[1.] $\phi ({{\bf{e}}_i},1) = V(i)$, $\forall i \in {\cal S}$;
		\item[2.] For ${\bf{b}} \notin \{ {{\bf{e}}_i}:i \in {\cal S}\} $,
		\begin{align*}
			\phi ({\bf{b}},1) = \mathop {\max }\limits_{a \in {\cal A}} \left\{ {\sum\limits_{s \in {\cal S}} {{\bf{b}}(s)r(s,a)}  + \beta {{\bf{b}}^\top}{P_a}V} \right\}.
		\end{align*}
	\end{itemize}
\end{lemma}
\begin{proof}
	The initial belief state being ${\bf e}_i$ means that the controller has observed that the initial state of the underlying MDP is $s_0=i$. Since $\rho=1$, the controller can receives $s_t$ for all $t$. As a result, the belief state always remains in the set $\{ {{\bf{e}}_i}:i \in {\cal S}\} $ and the belief MDP reduces to the original MDP. Hence $\phi ({{\bf{e}}_i},1) = V(i),\forall i \in {\cal S}$. Statement 2 follows immediately from the Bellman equation and statement 1. 
\end{proof}

Clearly, $\phi ({{\bf{e}}_i},1) = V(i) = J({{\bf{e}}_i})$ for all $i\in {\cal S}$. It is interesting that, for any ${\bf{b}} \notin \{ {{\bf{e}}_i}:i \in {\cal S}\} $, $\phi ({\bf{b}},1)$ given by Lemma \ref{lem:rho=1} is not equal to $J({\bf b})$ defined in \eqref{eq:Jtheta}. Both of them consider that the initial state of the underlying MDP follows a distribution ${\bf b}$. However, $J({\bf b})$ in the classical setting assumes that the controller can observe exactly the initial state $s_0$ before it computes the first action $a_0$; hence $a_0$ depends on the specific realization of  $s_0$ and is always optimal. By contrast, $\phi ({\bf{b}},1)$ corresponds to the case that the controller does not know the exact initial state and needs to determine $a_0$ only based on ${\bf b}$; consequently, $a_0$ is likely to be sub-optimal for a particular realization of $s_0$. Therefore, $\phi ({\bf{b}},1) \le J({\bf{b}})$ and the difference $J({\bf{b}}) - \phi ({\bf{b}},1)$ is the performance regret of the belief MDP generated by the uncertainty at the first time step.

The above discussion is useful for understanding the next theorem, where we provide a bound for the performance regret caused by state information losses.
\begin{theorem}  \label{thm:regretbound}
	For any $\mathbf{b} \in {\cal W}$, $\phi ({\bf{b}},\rho )$ is continuous at any $\rho\in (0,1)$. In addition, denote by $\pi_1$ an optimal policy for ${\cal B}(1)$. Then, for any $\rho\in (0,1]$,
	\begin{align*}
		\phi [1] - \phi [\rho ] \le {\left( {I - \beta T_\rho ^{{\pi _1}}} \right)^{ - 1}}\eta, 
	\end{align*}
	where $\eta :{\cal W} \to \mathbb{R}$ is given by (let $\tau  = P_{{\pi _1}({\bf{b}})}^\top{\bf{b}}$)
	\begin{align*}
		\eta ({\bf{b}}) = \beta (1 - \rho )\left[ {J(\tau ) - \phi (\tau ,1)} \right],\quad {\bf{b}} \in {\cal W}.
	\end{align*}
\end{theorem}
\begin{proof}
	See Appendix A.
\end{proof}

Basically, the upper bound for the performance regret is obtained by applying the optimal policy of ${\cal B}(1)$ (i.e., $\pi_1$) to the belief MDP ${\cal B}(\rho)$ and comparing the resulting value function with $\phi(\cdot,1)$. Start from a belief state ${\bf b}$, ${\cal B}(1)$ and ${\cal B}(\rho)$ generate the same instantaneous reward $R({\bf b}, \pi({\bf b}))$ when they are controlled by the same policy $\pi_1$; the performance regret comes from the second step. For ${\cal B}(1)$ under policy $\pi_1$, the belief state transitions to ${\bf e}_i$ with probability $\tau(i)$ for each $i\in {\cal S}$. By contrast, for ${\cal B}(\rho)$ under policy $\pi_1$, the belief state transitions to ${\bf e}_i$ with probability $\rho \tau(i)$ and transitions to $\tau$ with probability $1-\rho$. The regret comes from the latter transition, as illustrated above. We can think of the upper bound of $\phi[1] - \phi[\rho]$ stated in Theorem \ref{thm:regretbound}  as the expected total discounted regret of a Markov regret process (MRP). In particular, the MRP is governed by the transition matrix $T_\rho ^{{\pi _1}}$, and the regret at state ${\bf b}$ is $\eta({\bf b})$.

\section{Truncated Approximation} \label{sec:TA}
Although the belief MDP is much simpler than general POMDPs, it is still hard to solve exactly due to the countably infinite state space. We thus reformulate the IOMDP as a tree MDP, based on which we propose two finite-state approximations for finding near-optimal policies. In the following sections, we consider $\rho\in (0,1]$ to be fixed but arbitrary. We thus drop $\rho$ from notations to simplify the expressions. For example, $\phi ({\bf{b}},\rho )$ will be written as $\phi({\bf b})$ and the associated vector form will be written as $\phi$.

\subsection{Tree MDP}
We next reformulate the IOMDP to ease expositing of the finite-state approximations. As defined in \eqref{eq:belief}, the belief state at any time is determined by the sufficient history at the controller. In particular, a sufficient history, denoted by $h = (s,{u_{1} },{u_2}, \cdots ,{u_n})$, is a tuple consisting of a state $s\in {\cal S}$ and a sequence of $n$ ordered actions, where $n\in \{0,1,2,\cdots\}$ and $u_k\in {\cal A}$ for all $k$. When $n=0$, we will simply write $h=s$. 

{ Denote by ${\cal H}$ the set of all sufficient histories. This set can be hierarchically organized in a tree with $|{\cal S}|$ roots and infinitely many layers. In particular, each element of set  ${{\cal G}_0} : = {\cal S}$ corresponds to a root. For $n\in \{1,2,\cdots\}$, denote by ${\cal G}_n$ the set of sufficient histories consisting of 1 state followed by  $n$ actions. Then set ${\cal G}_n$ consists of the nodes at layer $n$ of the tree. Every sufficient history $h\in {\cal G}_n$ has $|{\cal A}|$ children, each corresponds to adding an extra action $a\in {\cal A}$ to the tuple, resulting in $(h,a)$. We then define $\{{\cal G}_n: n=0,1,2,\cdots\}$ recursively by
	\begin{align} \label{eq:Gn}
		{{\cal G}_{n + 1}} \buildrel \Delta \over = \left\{ {(h,a):h \in {{\cal G}_n},a \in {\cal A}} \right\}.
	\end{align}
	Fig.\ref{fig:tree} is an example of sufficient histories organized in a tree structure. Using the terminology in tree structures,  $h$ will be referred to as the parent node of $(h,a)$. Except for the root nodes (i.e., those in ${\cal G}_0$), each node in the tree has a unique parent. In addition, ${\cal G}_k$ and ${\cal G}_n$ are disjoint for any $n\neq k$ and ${\cal H} =  \cup _{k = 0}^\infty {{\cal G}_k}$. For any positive integer $n$, let ${\cal H}_n =  \cup _{k = 0}^n {{\cal G}_k}$ denote the set of nodes from layer 0 to layer $n$.
	
	According to the definition of belief states in \eqref{eq:belief}, there is a mapping $g:{\cal H} \to {\cal W}$ that determines the belief state for a given sufficient history. Specifically, for any $h=(s,{u_1},{u_2}, \cdots ,{u_n})  \in {\cal H}$,
	\begin{align} \label{eq:g(h)}
		g(h)  = \prod\limits_{k = 1}^n {P_{{u_k}}^\top} {{\bf{e}}_s} = P_{{u_n}}^\top P_{{u_{n - 1}}}^\top \cdots P_{{u_1}}^\top{{\bf{e}}_s},
	\end{align}
	where $\prod\nolimits_{k = 1}^n {P_{{u_k}}^\top}  = I$ for $n=0$. Note that $g$ is not necessarily a one-to-one mapping. It is possible that $g(h_1)=g(h_2)$ for $h_1\neq h_2$.
}
\begin{figure}[t]
	\centering
	\includegraphics[width=3.5in]{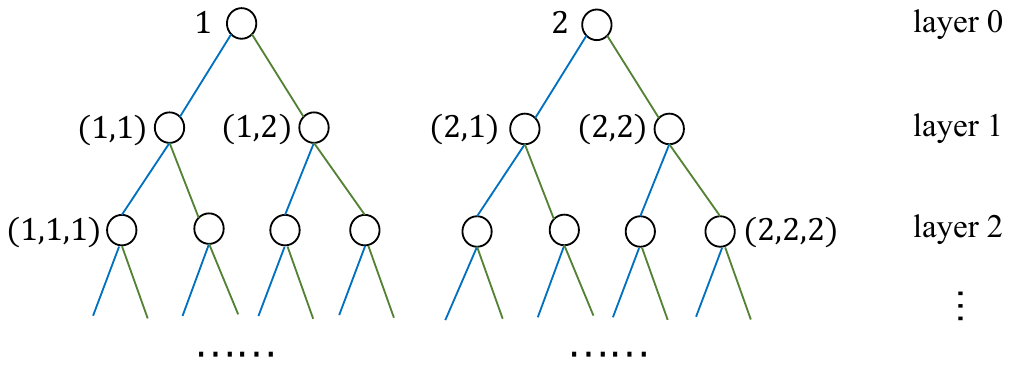}
	\caption{Sufficient histories organized in a tree structure (${\cal S}={\cal A}=\{1,2\}$).}
	\label{fig:tree}	
\end{figure}

Considering ${\cal H}$ as the state space, we can reformulate the IOMDP as an MDP denoted by ${\cal C}(\rho ) = ({\cal H},{\cal A},\rho ,D,{\bar R},\beta )$. In particular, the reward function is $\bar R(h,a): = R\left( {g(h),a} \right)$ and the transition kernel is given by (cf. eq.\eqref{eq:kernel_T})
\begin{align} \label{eq:kernel-D}
	D(h'|h,a) \buildrel \Delta \over = \Pr \left( {h'|h,a} \right) = \begin{cases}
		1 - \rho ,& {\text{if  }}h' = (h,a)\\
		\rho {[P_a^\top g(h)]_i},& {\text{if  }}h' = i \in {\cal S}\\
		0,&{\text{otherwise.}}
	\end{cases}
\end{align}
We call ${\cal C}(\rho)$ a {\bf tree MDP} because of the tree structure of its state space. To distinguish from the belief state, the state of the tree MDP at time $t$, denoted by $h_t\in {\cal H}$, will be referred to as the {\bf position state}. Let $\varphi :{\cal H} \to \mathbb{R}$ denote the optimal value function of ${\cal C}(\rho)$. We can easily identify the relationship between the two formulations of the IOMDP, as stated in the lemma below. The proof is provided in Appendix B for the sake of completeness.
\begin{lemma} \label{lem:2formulations}
	For the two formulations ${\cal B}(\rho)$ and ${\cal C}(\rho)$,
	\begin{itemize}
		\item [1.] $\phi({\bf b})=\varphi(h)$ if ${\bf b} = g(h)$;
		\item [2.] $\varphi(h) = \varphi(h')$ for any $h,h'\in {\cal H}$ satisfying $g(h) = g(h')$;
		\item [3.] Denote by $\pi$ and $\mu$ the optimal policies for ${\cal B}(\rho)$ and ${\cal C}(\rho)$, respectively. Then $\pi({\bf b})=\mu(h)$ if ${\bf b} = g(h)$.
	\end{itemize}
\end{lemma}

The two formulations, although similar, have different advantages. As shown in the previous section, the belief MDP formulation allows mathematical analysis of the optimal value function, thus establishing fundamental properties of the IOMDP. By contrast, we find it convenient to develop finite-state approximations and compute near-optimal policies for the problem using the tree MDP formulation.

\subsection{Truncated Approximation}
Our first finite-state approximation for  ${\cal C}(\rho)$ is inspired by a simple fact: starting from the root layer, the position state $h_t$ can arrive at layer $n$ only if the controller suffers $n$ consecutive transmission failures. The probability of such an event is $(1-\rho)^n$, which decreases exponentially as $n$ increases. When the probability of $h_t$ visiting ${\cal G}_n$ is small enough, the expected performance regret caused by taking a sub-optimal action in $h_t\in {\cal G}_n$ should be negligible. 

Let us consider the following  $L$-truncated approximation.
\begin{definition}[$L$-Truncated Approximation, TA($L$)]
	Given a positive integer $L$, the $L$-truncated approximation of ${\cal C}(\rho)$ is an MDP denoted by ${\cal C}{_L}(\rho ) = ({{\cal H}_L},{\cal A},\rho ,{D_L},\bar R,\beta )$. In particular, the state space of ${\cal C}_L(\rho)$ is ${{\cal H}_L} =  \cup _{n = 0}^L{{\cal G}_n}$. The action space ${\cal A}$, SIRP $\rho$, reward function $\bar R$, and discount factor $\beta$ are identical to those of ${\cal C}(\rho)$. The transition kernel $D_L$ is defined as follows:
	\begin{itemize}
		\item [(1)] For any $h \in {{\cal G}_n}$ with $n \le L - 1$, ${D_L}(h'|h,a) = D(h'|h,a)$ for all $h'\in {\cal H}_L$ and $a\in {\cal A}$;
		\item [(2)] For any $h \in {{\cal G}_L}$, let $\tau_a = (h,a)$ and ${\bf b}_a = g(\tau_a)$, where $a\in {\cal A}$. Then
		\begin{align} \label{eq:kernel-DL}
			D_L(h'|h,a) \buildrel \Delta \over = \Pr \left( {h'|h,a} \right) = \begin{cases}
				1 - \rho ,& {\text{ if  }}h' =h\in {\cal G}_L\\
				\rho {\bf b}_a(i),& {\text{ if  }}h' = i \in {\cal S}\\
				0,&{\text{    otherwise.}}
			\end{cases}
		\end{align}
	\end{itemize}
\end{definition}

\begin{figure}[t]
	\centering
	\includegraphics[width=3.3in]{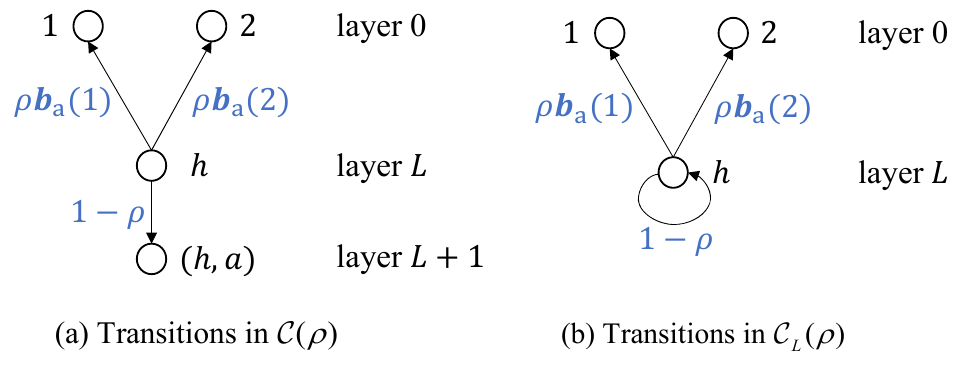}
	\caption{Transition probabilities of taking action $a\in {\cal A}$ in $h\in {\cal G}_L$ (${\cal S}=\{1,2\}$).}
	\label{fig:tal}	
\end{figure}

The state space of TA($L$) consists of position states at the first $L+1$ layers. Therefore, for any $h \in {{\cal G}_L}$, $(h,a) \in {{\cal G}_{L+1}}$ is not a state of ${\cal C}_L(\rho)$. We thus redefine the transition probabilities for $h \in {{\cal G}_L}$ in ${\cal C}_L(\rho)$. Fig.\ref{fig:tal} compares the transition probabilities, in ${\cal C}(\rho)$  and  ${\cal C}_L(\rho)$, of taking a particular action in $h \in {{\cal G}_L}$. TA($L$) can be viewed as an adaptation for practical operations. That is, the controller does not update the position state (and the associated belief state) after it experiences more than  $L$ consecutive transmission failures until it receives a new state observation. We will use  $\varphi_L(\cdot)$ to denote the optimal value function of  ${\cal C}_L(\rho)$. It turns out that $\varphi_L(h)$ is a good approximator of $\varphi(h)$ for part of $h\in {\cal H}_L$. The theorem below provides a bound for the approximation error of TA($L$).
\begin{theorem} \label{thm:TAL-bound}
	For any $\rho\in (0,1]$ and positive integer $L$,
	\begin{align*}
		&\mathop {\max }\limits_{h \in {{\cal G}_k}} \left| {{\varphi _L}(h) - \varphi (h)} \right| \le \\
		&{\delta _L}\left( {\frac{{{\beta ^{L + 1}}{{(1 - \rho )}^{L + 1}}}}{{1 - \beta }} + \sum\limits_{t = L - k}^{L - 1} {{\beta ^{t + 1}}{{(1 - \rho )}^{t + 1}}} } \right),{\rm{  }}0\le k \le L,
	\end{align*}
	where ${\delta _L} = \mathop {\max }\limits_{x \in {{\cal G}_L} \cup {{\cal G}_{L + 1}}} \varphi (x)$.
\end{theorem}
\begin{proof}
	See Appendix B.
\end{proof}

{ The upper bound presented in Theorem \ref{thm:TAL-bound} is insightful. We refer to $k$ as the layer index of $h$ if $h\in {\cal G}_k$. Notably, the error bound in Theorem \ref{thm:TAL-bound} is an increasing function of $k\in \{0,1,\cdots,L\}$, indicating that the TA($L$) offers a more accurate value approximation for position states with smaller layer indices. The monotonicity of the error bound w.r.t. the layer index can be understood through the lens of perturbation analysis \cite{caostochastic}. When the transition kernel is changed from $D$ to $D_L$, a perturbation could occur only if the position state is in the truncated layer (i.e., layer $L$). However, starting from the $k$-th layer, the position state can arrive at the truncated layer only if the controller experiences at least $L-k+1$  consecutive transmission failures, with the associated probability being $(1-\rho)^{L-k+1}$; its impact is further discounted by a factor of $(1-\beta)^{L-k+1}$. Consequently, the lower the initial layer index, the less likely a perturbation occurs, leading to a smaller approximation error.
	When $k=0$, the error bound reduces to ${\delta _L}{\beta ^{L + 1}}{(1 - \rho )^{L + 1}}/(1 - \beta )$, which can be arbitrarily small if $L$  is large enough. While for $h\in {\cal G}_L$, the approximation error is dominated by the term  $\beta (1 - \rho ){\delta _L}$, which may be non-negligible.
	
}

{ The optimal policy for the tree MDP  ${\cal C}(\rho)$ is determined by the Bellman equation
	\begin{align} \label{eq:bellman-C}
		\varphi (h) = \max_{a\in \mathcal{A}} Q(h,a) ,
	\end{align}
	where $Q: {\cal H} \times {\cal A} \to \mathbb{R}$ is the state-action value function:
	\begin{align*}
		Q(h,a) = \bar R(h,a) + \beta \sum\limits_{h' \in {\cal H}} {D\left( {h'|h,a} \right)\varphi (h')}.
	\end{align*}
	An action $a\in \mathcal{A}$ is said to be $\epsilon$-optimal for ${\cal C}(\rho)$ in position state $h$ if $\left| {Q(h,a) -\varphi(h)} \right| \le \epsilon$. 
	The following fact can be easily verified:
	
	\noindent {\bf Fact 1}: {\it For any $h\in {\cal H}$ in the tree MDP ${\cal C}(\rho)$, define
		\begin{align*}
			{\epsilon_h} \buildrel \Delta \over = \min \left\{ {\varphi(h) - Q(h,a)|a \in {\cal A}{\text{ is non-optimal}}} \right\}.
		\end{align*}
		If an action $a\in \mathcal{A}$ is $\epsilon$-optimal in position state $h$ with $\epsilon \le \epsilon_h$, then action $a$ is guaranteed to be the optimal for ${\cal C}(\rho)$ in position state $h$.
	}
	
	According to \eqref{eq:kernel-D}, for any $h\in \mathcal{G}_k$, $D(h'|h,a)>0$ only if $h'\in {{\cal G}_0} \cup {{\cal G}_{k + 1}}$. Therefore, if $\varphi_L(h')$ is sufficiently close to $\varphi(h')$ for all $h' \in {{\cal G}_0} \cup {{\cal G}_{k + 1}}$, we can obtain the optimal action for every $h\in {\cal G}_k$ by replacing $\varphi(\cdot)$ with $\varphi_L(\cdot)$ in the Bellman equation. 
}

Denote by $\mu_L$ an optimal policy for ${\cal C}_L(\rho)$ and $\mu_L(h)$ the associated action for $h\in {\cal H}_L$. According to Theorem \ref{thm:TAL-bound}, given any $\epsilon>0$ and positive integer $n$, there exists a finite integer $L_n \ge n$ such that $\left| {{\varphi _L}(h) - \varphi (h)} \right| < \epsilon$ for all $L\ge L_n$ and $h \in  \cup _{k = 0}^n{{\cal G}_k}$. Consequently, for any $L\ge L_n$ and $h \in  \cup _{k = 0}^n{{\cal G}_k}$, $\mu_L(h)$ is $\epsilon$-optimal for ${\cal C}(\rho)$. With $\mu_L$, we can derive a policy for ${\cal C}(\rho)$ to control the underlying MDP in the presence of state information losses. In particular,
\begin{definition}[TA($L$) Policy]
	For any positive integer $L$, define the TA($L$) policy for ${\cal C}(\rho)$ as follows:
	\begin{itemize}
		\item [(1)] For $h\in {\cal H}_L$, the TA($L$) policy takes action $\mu_L(h)$;
		\item [(2)] For $h\notin {\cal H}_L$, denote by $h'\in {\cal G}_L$ the ancestor node of $h$ in ${\cal G}_L$. Then the TA($L$) policy takes action $\mu_L(h')$ in position state $h$.
	\end{itemize}
\end{definition}
Intuitively, the larger the value of $L$, the closer the TA($L$) policy is to the optimal policy.
More analysis about the TA($L$) policy will be provided in the next section.

\section{High-order Truncated Approximation} \label{sec:HOTA}
Although TA($L_n$) can be used to determined an $\epsilon$-optimal action for each $h \in  \cup _{k = 0}^n{{\cal G}_k}$, we typically have $L_n>n$ for small $\epsilon$. Hence a relatively large $L_n$ is needed to obtain a satisfactory policy. This makes the method inefficient and usually prohibitively complex since the state space of TA($L$) increases exponentially with $L$:
\begin{align} \label{eq:HL}
	\left| {{{\cal H}_L}} \right| = \left| {{{\cal S}}} \right|\frac{{{{\left| {\cal A} \right|}^{L + 1}} - 1}}{{\left| {\cal A} \right| - 1}}.
\end{align}
As a result, TA($L$) is limited to IOMDPs with large $\rho$ and small action spaces. We next put forth a more efficient and scalable approach to finding near-optimal policies for ${\cal C}(\rho)$. The following concepts are important for developing the approach.
\begin{definition}
	Given a policy $\mu$, a position state of ${\cal C}(\rho)$ is said to be reachable under policy $\mu$ if it can be visited within a finite time with a positive probability. A position state is redundant if it is not reachable.
\end{definition}

As discussed around \eqref{eq:Gn}, every position state $h$ has $|{\cal A}|$ children, among which $(h,a)$ can be visited only if the controller takes action $a$ in $h_t=h$ and fails to receive $s_{t+1}$. Therefore, using a deterministic policy, each position state has only one reachable child. That is, each layer in the tree contains exactly $|{\cal S}|$  reachable position states. If we fix the optimal action for  $h$ to be, say $\mu(h)$, then $(h,a)$ will never be visited if $a\neq \mu(h)$ (except for the initial position states). That is, under the optimal policy, $(h,\mu(h))$ is reachable, and the remaining children of  $h$ are redundant. As far as computing the optimal policy is concerned, we can omit the redundant position states and remove them from the state space.

While the number of reachable position states at each layer is fixed (i.e, $|{\cal S}|$), the total number of position states increases exponentially over layers---most of them are redundant. This is the source of the inefficiency and poor scalability of TA($L$). The computation complexity can be significantly reduced if we can remove the redundant position states when computing a policy.

Let us consider the following idea to compute $\epsilon$-optimal actions for $h \in  \cup _{k = 0}^n{{\cal G}_k}$. Instead of solving for TA($L_n$) with a large $L_n$, we first select an $L_0 < L_n$ so that $L_0$ is large enough for computing $\epsilon$-optimal actions for all $h\in {\cal G}_0$ by solving TA($L_0$). Upon determining an $\epsilon$-optimal action for every $h\in {\cal G}_0$, we can identify the reachable position states in ${\cal G}_1$. The next step is to compute  $\epsilon$-optimal actions for the reachable position states in ${\cal G}_1$---we do not compute actions for the redundant position states since they will never be reached. For this purpose, we construct a modified TA($L_0+1)$, which will be referred to as TA($1,L_0$), by removing the redundant position states in ${\cal G}_1$ and their descendants at other layers from the state space of TA($L_0+1$). Solving for TA($1,L_0$) is considerably easier than  TA($L_0+1$) due to the reduction of state space, but it still provides  $\epsilon$-optimal actions for the reachable position states in ${\cal G}_1$.
Doing so repeatedly, we can compute  $\epsilon$-optimal actions for reachable position states layer by layer. We call this method high-order truncated approximation, where the $n$-th order truncated approximation refers to repeating the above procedure $n$  times and will be denoted by TA($n,L$). We will justify this method after providing a formal definition below.

We formally define TA($n,L$) in a recursive way for $n\in \{0,1,2,\cdots \}$ and a fixed $L\in \{1,2,3,\cdots \}$. In particular, TA($n,L$) is an MDP denoted by ${\cal C}^n_L(\rho)$ with ${\cal C}^0_L(\rho) \triangleq {\cal C}_L(\rho)$. Let $\mu^n_L$ stand for the optimal policy for ${\cal C}^n_L(\rho)$. As discussed before, with a proper $L$, $\mu^0_L(h) = \mu_L(h)$ is an $\epsilon$-optimal action for ${\cal C}(\rho)$ in position state $h\in {\cal G}_0$. Our aim is to define a series of $\{{\cal C}^n_L(\rho):n=1,2,3,\cdots \}$ such that, for every reachable $h$ in ${\cal G}_n$, $\mu^n_L(h)$ is an $\epsilon$-optimal action for ${\cal C}(\rho)$. Meanwhile, the redundant position states in ${\cal G}_n$ and their descendants at other layers will be excluded from the state space of ${\cal C}^n_L(\rho)$ to reduce computation complexity.

\begin{figure}[t]
	\centering
	\includegraphics[width=2.8in]{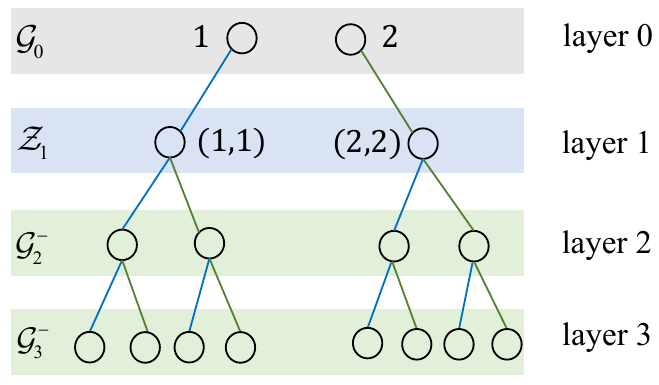}
	\caption{State space of TA(1,3), ${\cal S}={\cal A}=\{1,2\},\mu^0_L(1)=1, \mu^0_L(2)=1$. The redundant position states at layer 1 and their descendants at layers 2 and 3 are removed.}
	\label{fig:hota}	
\end{figure}

Mathematically, ${\cal C}_L^n(\rho ) = ({\cal H}_L^n,\{ {\cal A}_h^n\} ,\rho ,D_L^n,\bar R,\beta )$ is an MDP with state space
\begin{align*}
	{\cal H}_L^n = \left( {\mathop  \cup \limits_{i = 0}^n {{\cal Z}_i}} \right) \cup \left( {\mathop  \cup \limits_{i = n + 1}^{L + n} {\cal G}_i^ - } \right),
\end{align*}
where ${{\cal Z}_i} \subseteq {{\cal G}_i}$ is the set of reachable position states in  ${\cal G}_i$ under the policy $\mu^{i-1}_L$:
\begin{align*}
	{{\cal Z}_i} \buildrel \Delta \over =  \begin{cases}
		{{\cal G}_0},& i = 0,\\
		\left\{ {\left( {h,\mu _L^{i - 1}(h)} \right):h \in {{\cal Z}_{i - 1}}} \right\},& 1 \le i \le n,
	\end{cases}
\end{align*}
and ${\cal G}_{n + i}^ -  \subset {{\cal G}_{n + i}}$ is the set of descendants of ${\cal Z}_n$ at layer $(n+i)$:
\begin{align*}
	{\cal G}_{n + i}^ -  \buildrel \Delta \over = \begin{cases}
		\left\{ {\left( {h,a} \right):a \in {\cal A},h \in {{\cal Z}_n}} \right\},& i = 1,\\
		\left\{ {(h,a):a \in {\cal A},h \in {\cal G}_{n + i - 1}^ - } \right\},&i > 1.
	\end{cases}
\end{align*}
Fig. \ref{fig:hota} gives an example of the state space of TA($n,L$) for $n=1, L=3$. The action space for any $h\in {\cal H}^n_L$ is given by
\begin{align}
	{\cal A}_h^n = \begin{cases}
		\left\{ {\mu _L^{n - 1}(h)} \right\},& {\text{if }}h \in \mathop  \cup \limits_{i = 0}^{n - 1} {{\cal Z}_i}\\
		{\cal A},&{\text{otherwise.}}
	\end{cases} 
\end{align}
The first line of the above definition means to fix the action for the already determined reachable states at layers 0 to $n-1$. With this restriction, we have $\mu _L^n(h) = \mu _L^{n - 1}(h)$ for all $h \in  \cup _{i = 0}^{n - 1}{{\cal Z}_i}$. The SIRP $\rho$, reward function $\bar R$, and discount factor $\beta$ are identical to those of ${\cal C}(\rho)$. The transition kernel $D^n_L$ is defined as follows:
\begin{itemize}
	\item [(1)] For $h \in \left( { \cup _{i = 0}^n{{\cal Z}_i}} \right) \cup \left( { \cup _{i = n + 1}^{L + n - 1}{\cal G}_i^ - } \right)$, $D_L^n(h'|h,a) = D(h'|h,a)$ for any $h'\in {\cal H}^n_L$ and $a\in {\cal A}^n_h$;
	\item [(2)] For $h \in {\cal G}_{L + n}^ - $, let $\tau_a = (h,a)$ and ${\bf b}_a = g(\tau_a)$, where $a\in {\cal A}$. Then
\end{itemize}
\begin{align} \label{eq:kernel-DLn}
	D^n_L(h'|h,a) \buildrel \Delta \over = \Pr \left( {h'|h,a} \right) = \begin{cases}
		1 - \rho ,& {\text{ if  }}h' =h\in {\cal G}_{L + n}^ -\\
		\rho {\bf b}_a(i),& {\text{ if  }}h' = i \in {\cal S}\\
		0,&{\text{    otherwise.}}
	\end{cases}
\end{align}

This completes the definition of ${\cal C}^n_L(\rho)$. Note that $|{{\cal Z}_i}| = |{{\cal G}_0}|$ for all $1\le i\le n$ and that ${\cal H}_L^n \subseteq {{\cal H}_{L + n}}$ can be determined with $\mu^{n-1}_L$. Starting from $n=0$ and ${\cal Z}_0 = {\cal G}_0$, we compute the optimal policy $\mu^n_L$ for ${\cal C}^n_L(\rho)$ and then use $\mu^n_L$ and ${\cal Z}_n$ to identify ${\cal Z}_{n+1}$ (the set of reachable position states in ${\cal G}_{n+1}$); then the descendants of ${\cal Z}_{n+1}$ at the following $L-1$ layers (i.e., ${\cal G}_{n + i}^ - $ for $2\le i \le L+1$) can be identified. On this basis, we construct ${\cal C}^{n+1}_L(\rho)$ to compute the actions for position states in ${\cal Z}_{n+1}$. Doing so iteratively with a proper $L$  yields the optimal actions of the reachable position states layer by layer (see Theorem \ref{thm:policies}). We can stop when $(1-\rho)^{n+1}$ is small enough and control ${\cal C}(\rho)$ using the TA($n,L$) policy derived from $\mu^n_L$:
\begin{definition}[TA($n,L$) Policy]
	For any positive integers $L$ and $n$, define the TA($n,L$) policy for ${\cal C}(\rho)$ as follows:
	\begin{itemize}
		\item [(1)] For $h\in {\cal H}^n_L$, the TA($n,L$) policy takes action $\mu^n_L(h)$;
		\item [(2)] For $h\notin {\cal H}^n_L$, denote by $h'\in {\cal G}_{L+n}$ the ancestor node of $h$ in ${\cal G}_{L+n}$. Then the TA($n,L$) policy takes action $\mu^n_L(h')$ in position state $h$.
	\end{itemize}
\end{definition}

We next establish the relationship between ${\cal C}^n_L(\rho)$ and ${\cal C}_{L+n}(\rho)$, and then analyze the TA($L$) and TA($n,L$) policies. Denote by $\varphi _L^n:{\cal H}_L^n \to \mathbb{R} $ the optimal value function of ${\cal C}^n_L(\rho)$. The following lemma is useful.
\begin{lemma}  \label{lem:CLn}
	For any fixed $L$, if ${\cal C}^n_L(\rho)$ and ${\cal C}_{L+n}(\rho)$ have the same optimal action for every $h \in  \cup _{i = 0}^{n - 1}{{\cal Z}_i}$, then ${\varphi _{L + n}}(h) = \varphi _L^n(h)$ for any $h\in {\cal H}^n_L$.
\end{lemma}
\begin{proof}
	See Appendix C.
\end{proof}

According to Fact 1, define
\begin{align*}
	\epsilon(n) \triangleq \min \{\epsilon_h : h\in {\cal H}_n \}.
\end{align*}
We show in the following theorem that, with a proper $L$, solving for ${\cal C}^n_L(\rho)$ yields $\epsilon(n)$-optimal (thus optimal) actions for all reachable position states at layers 0 to $n$. Recall that $\mu^n_L$ and $\mu_{L+n}$ denote the optimal policies for ${\cal C}^n_L(\rho)$ and ${\cal C}_{L+n}(\rho)$, respectively.

\begin{theorem} \label{thm:policies}
	For any positive integer $n$, given an $L$ satisfying
	\begin{align*}
		\delta (n,L)\left( {\frac{{{\beta ^{L + 1}}{{(1 - \rho )}^{L + 1}}}}{{1 - \beta }} + {\beta ^L}{{(1 - \rho )}^L}} \right) \le \epsilon(n),
	\end{align*}
	where $\delta (n,L) = \mathop {\max }\limits_{x \in {{\cal H}_{L + n}}} \varphi (x)$. Then,
	\begin{itemize}
		\item [(1)] $\mu_{L+n}(h)$ is an optimal action for any $h \in  \cup _{k = 0}^n{{\cal G}_k}$ in ${\cal C}(\rho)$.
		\item [(2)] $\mu^n_L(h)$ is an optimal action for any $h \in  \cup _{k = 0}^n{{\cal Z}_k}$ in ${\cal C}(\rho)$.
	\end{itemize}
\end{theorem}
\begin{proof}
	See Appendix C.
\end{proof}

Theorem \ref{thm:policies} shows that both TA($L+n$) and TA($n,L$) can be used to obtain optimal actions for reachable position states at layers 0 to $n$. However, it is easy to see that TA($n,L$)  is much more efficient than TA($L+n$). The cardinality of the state space of ${\cal C}^n_L(\rho)$  is increasing linearly with $n$, that is,
\begin{align*}
	\left| {{\cal H}_L^n} \right| = \left| {{{\cal G}_0}} \right|\left( {\frac{{{{\left| {\cal A} \right|}^{L + 1}} - 1}}{{\left| {\cal A} \right| - 1}} + n} \right),
\end{align*}
where $ {{{\cal G}_0}} = {\cal S} $. By contrast, \eqref{eq:HL} shows that $|{{\cal H}_{L + n}}|$ increases exponentially with $n$ for any fixed $L$. Although solving ${\cal C}^n_L(\rho)$ needs to solve ${\cal C}^k_L(\rho)$ for all $k\le n$, the overall computation complexity of TA($n,L$) is still significantly lower than that of TA($L$).
In addition, if we use a value iteration algorithm to solve ${\cal C}^n_L(\rho)$, then the optimal value function of ${\cal C}^{n-1}_L(\rho)$ is a good starting point for the value iteration. Since ${\varphi}^{n-1}_L(h)$ is close to ${\varphi}^n_L(h)$ for $h \in  \cup _{i = 0}^{n - 1}{{\cal Z}_i}$. It is well-known that the value iteration converges quickly if the starting point is near the optimum. 

Since TA$(L+n)$ and TA($n,L$) policies may take suboptimal actions for $h\notin\mathcal{H}_n$, we analyze their suboptimality in the following theorem.
{Given the MDP $\mathcal{C}(\rho)$, we say that a policy $\mu$ is $\varepsilon$-optimal for an initial state $h_0$ if its corresponding value function, denoted as $\varphi^\mu$, satisfies the condition $|\varphi^\mu(h_0) - \varphi(h_0)| \leq \varepsilon$.}

\begin{theorem} \label{thm:eps-optimal}
	Given an IOMDP $\mathcal{C}(\rho)$ and any $\varepsilon>0$, both TA($L+n$) and TA($n,L$) policies are $\varepsilon$-optimal for any initial state $h_0\in \mathcal{G}_0$ if $n$ and $L$ satisfy the following:
	\begin{align*}
		n\ge \frac{\log \varepsilon(1-\beta)/(2K-\varepsilon\beta\rho) }{\log \beta(1-\rho)} -1,\\
		L\ge \frac{\log\epsilon(n)(1-\beta)^2/K(1-\beta \rho) }{\log \beta(1-\rho)},
	\end{align*}
	where $K=\max_{h,a}|\bar{R}(h,a)|$ is the maximum absolute value of the reward function.
\end{theorem}
\begin{proof}
	See Appendix C.
\end{proof}

A limitation of the TA($L+n$) and TA($n,L$) policies is that, even with $n$ and $L$ satisfying Theorem \ref{thm:eps-optimal}, they are guaranteed to be $\varepsilon$-optimal only when the initial position state is in $\mathcal{G}_0$. In other words, if the initial position state $h_0\notin \mathcal{G}_0$, both policies may exhibit a regret (i.e., $|\varphi^u(h_0)-\varphi(h_0)|$) that exceeds $\varepsilon$. Fortunately, this limitation can be easily remedied by adding an additional node to the root layer. For any initial belief state $\mathbf{b}_0\notin \{\mathbf{e}_i:i\in\mathcal{S} \}$, we can define a virtual node $h_0$ in the root layer. Subsequently, we define its descendants at each layer using \eqref{eq:Gn} and associate each descendant with a belief using \eqref{eq:g(h)}, simply replacing $\mathbf{e}_s$ with $h_0$. Doing so allows TA($L+n$) and TA($n,L$) policies to achieve $\varepsilon$-optimal value for any initial belief state.

\section{Nested Value Iteration} \label{sec:NVI}
We have shown in the previous section that solving for TA($L$) and TA($n,L$) could offer near-optimal policies for the IOMDP. This section presents a variant of the value iteration algorithm, called nested value iteration (NVI), to compute the optimal policy for TA($L$) efficiently. It also applies to TA($n,L$).

The algorithm is inspired by the special structure of the tree MDP. The significance of position states at the root layer (i.e., ${\cal G}_0$) in the Bellman equation is evident, as their values directly influence the values of all position states, as indicated by \eqref{eq:bellman-C}. 
Intuitively, for a non-zero $\rho$, the larger the layer index of $h$, the weaker the influence of  $h$’s value on the overall value function. The basic idea of NVI is to update the values of important position states more frequently than the less important ones, as shown in Algorithm \ref{al:NVI}.

In principle, based on the tree structure of the state space of  ${{\cal C}_L}(\rho )$, we define such a collection of $d$ nested sets $\{\mathcal{X}_i:1\le i\le d \}$ that the largest set is the entire state space (i.e., $\mathcal{X}_d=\mathcal{H}_L$) and $\mathcal{X}_i\subseteq \mathcal{X}_{i+1}$ for $i\in \{1,2,\cdots,d-1\}$.
As commented in Algorithm \ref{al:NVI}, NVI consists of outer and inner iterations. Each outer iteration consists of $d$ inner iterations, where the  $l$-th inner iteration only updates the values of position states in set ${\cal X}_{d-l+1}$ ($1 \le l\le d $). Consequently, position states in ${\cal X}_1$ are updated in every inner iteration, while position states in ${\cal X}_d - \mathcal{X}_{d-1}$ are updated once in every outer iteration (i.e., every  $d$ inner iterations).

\begin{algorithm}[t]
	\caption{Nested Value Iteration (NVI) }
	\label{al:NVI}
	\KwIn{${{\cal C}_L}(\rho ) = ({{\cal H}_L},{\cal A},\rho ,D_L ,\bar R,\beta )$, a positive integer $d$. \\		
	}
	\textbf{Initialization:} select the initial value function $\psi^0$, specify $\sigma >\epsilon>0$, and set $n=0$. Define ${{\cal X}_d} :=  \mathcal{H}_L$ and $\mathcal{X}_i\subseteq \mathcal{X}_{i+1}$ for $i\in\{1,2,\cdots,d-1 \}$.   \\	
	\While{$\sigma>\epsilon$}{					
		// for loop 1: {\bf outer iteration} \\
		\For{ $l=0$ to $d-1$}{
			// for loop 2: {\bf inner iteration} \\
			\For{each $h\in {\cal X}_{d-l}$}{
				${\psi ^{n + 1}}(h) = \mathop {\max }\limits_{a \in {\cal A}} \left\{ {\bar R(h,a) + \beta \sum\limits_{h' \in {{\cal H}_L}} {{D_L}\left( {h'|h,a} \right){\psi ^n}(h')} } \right\}$
			}
			${\psi ^{n + 1}}(h) = {\psi ^n}(h), \forall h \notin {{\cal X}_{d - l}}$ \\	
			\If{$l=0$}{
				$\sigma  = {\left\| {{\psi ^{n + 1}} - {\psi ^n}} \right\|_\infty }$ \quad // max norm	\\
				\If{$\sigma\le \epsilon$}{
					{\bf Return} ${\psi ^{n + 1}}$ 
				}
			}
			
			$n=n+1$ \\
		}		
	}
\end{algorithm}

It is worth noting that the NVI allows for multiple definitions of $\{\mathcal{X}_i\}$.
In fact, we expect that NVI would be effective in a broad class of MDPs, provided that $\{\mathcal{X}_i\}$ is properly defined.
For the IOMDP we focused on in this paper, we have identified several definitions that make the NVI more efficient than the standard value iteration. Among these definitions, the following one in particular demonstrates exceptional performance:
\begin{align} \label{eq:nestedset}
	\mathcal{X}_d=\mathcal{H}_L, \mathcal{X}_i=\mathcal{H}_1 \text{ for all }1\le i\le d-1,
\end{align}
where the positive integer $d$ is an adjustable parameter.
The motivation behind this definition is to leverage the importance of the root layer in enhancing the overall algorithm performance. There are two key factors contributing to this: (i) position states at the root layer hold a pivotal role in the Bellman equation, exerting a significant influence; (ii) conducting additional operations on this layer entails minimal computational cost, as the root layer represents only a small subset of the entire state space.

\begin{theorem} \label{thm:NVI}
	Let $\{\psi^n\}$ denote the sequence of functions generated by the nested value iteration algorithm with input ${{\cal C}_L}(\rho )$ and a positive integer $d$. Then
	\begin{itemize}
		\item [(1)] For any collection of nested sets $\{\mathcal{X}_i \}$ that satisfies $\mathcal{X}_d=\mathcal{H}_L$,  $\psi^n$ converges in max norm to the optimal value function $\varphi_L$;
		\item [(2)] If the collection of nested sets $\{\mathcal{X}_i \}$ is defined by \eqref{eq:nestedset}, then 
		\begin{align*}
			&\left|\left| {{\psi ^{kd + 1}} - {\psi^{kd}}} \right|\right|_\infty \le   \\
			&\left[ {\frac{{\beta (1 - \rho )}}{{1 - \beta \rho }} + \frac{{{\beta ^{d}}{\rho ^{d}}(1 - \beta )}}{{1 - \beta \rho }}} \right] \left|\left| {{\psi ^{(k-1)d + 1}} - {\psi^{(k-1)d}}} \right|\right|_\infty
		\end{align*}
		for $k\in \{1,2,3,\cdots\}$, where $||\cdot||_\infty$ denotes the max norm. 
	\end{itemize}
\end{theorem}
\begin{proof}
	See Appendix D.
	
\end{proof}

Theorem \ref{thm:NVI} shows how the additional updates over the root layer accelerate the convergence. Note that for any $d>1$, 
\begin{align*}
	{\frac{{\beta (1 - \rho )}}{{1 - \beta \rho }} + \frac{{{\beta ^{d}}{\rho ^{d}}(1 - \beta )}}{{1 - \beta \rho }}} < {\frac{{\beta (1 - \rho )}}{{1 - \beta \rho }} + \frac{{\beta\rho(1 - \beta )}}{{1 - \beta \rho }}} = \beta.
\end{align*}
It is well known that the standard value iteration exhibits a contraction factor of $\beta$ across iterations. Hence Theorem \ref{thm:NVI} establishes that the NVI demonstrates a reduced contraction factor (across outer iterations) compared to the standard value iteration. The additional operations involved in an outer iteration, in contrast to an iteration in the standard value iteration, can usually be considered negligible as $\mathcal{H}_1$ is only a small subset of $\mathcal{H}_L$.   

Formally, we demonstrate NVI's efficiency by conducting an analysis of its computational complexity. 
Following convention, we say that a policy $\mu$ is $\varepsilon$-optimal if its value function $\psi^\mu$ satisfies $|\psi^\mu(h)-\varphi_L(h)|\le \varepsilon$ for all $h\in \mathcal{H}_L$.
It is well known that the standard value iteration can find an $\varepsilon$-optimal policy in time $O\left(|\mathcal{H}_L|^2 |\mathcal{A}| \frac{\log 1/\varepsilon(1-\beta)}{1-\beta} \right)$, which can be written as follows by substituting the expression of $\mathcal{H}_L$ given by \eqref{eq:HL}:
\begin{align*}
	O\left(|\mathcal{S}|^2 |\mathcal{A}|^{2L+1} \frac{\log 1/(\varepsilon(1-\beta))}{1-\beta} \right).
\end{align*}
Based on Theorem \ref{thm:NVI}, we can obtain the following result. It shows that the proposed algorithm enhances the standard value iteration by reducing the number of iterations needed for convergence, with a cost of a slight increase in the number of operations within each iteration.  

\begin{theorem} \label{thm:complexity}
	Given ${{\cal C}_L}(\rho )$ with root layer  $\mathcal{G}_0=\mathcal{S}$, the NVI algorithm with $\{\mathcal{X}_i \}$ defined by \eqref{eq:nestedset} can be implemented to obtain an $\varepsilon$-optimal policy in time
	\begin{align*}
		O\left(|\mathcal{S}|^2\left(d|\mathcal{A}|^3+|\mathcal{A}|^{2L+1}\right)  \frac{\log 1/\varepsilon(1-\beta) }{1-\beta} \frac{1-\beta\rho}{1-\beta^{d}\rho^{d}} \right).
	\end{align*}
\end{theorem}
\begin{proof}
	See Appendix D.
	
\end{proof}

The above theorem demonstrates the dependency of the proposed algorithm on all related factors, including the inherent factors of the tree MDP $\mathcal{C}_L(\rho)$ ($|\mathcal{S}|$, $|\mathcal{A}|$, $\beta$, $L$ and $\rho$) and the adjustable parameter $d$ in our approach. 
It illustrates the impact of the parameter $d$ on the algorithm's efficiency: as $d$ increases, the number of outer iterations required for convergence decreases approximately at a rate of $(1-\beta\rho)/(1-\beta^d\rho^d)$, while the number of additional operations within each outer iteration grows linearly with $d$. {The analysis shows that the optimal efficiency of the NVI algorithm is achieved by choosing $d$ to minimize the term $\left(d|\mathcal{A}|^3+|\mathcal{A}|^{2L+1}\right)/(1-\beta^{d}\rho^{d})$.  }

Finally, extending NVI to solve TA($n,L$) is straightforward. For ${\cal C}^n_L (\rho)$ with $n\ge 1$, the action of $h \in  \cup _{i = 0}^{n - 1}{{\cal Z}_i}$ is already determined, hence there is no need to update the value of these position states. The values of root position states are identical to that of  ${\cal C}^0_L (\rho)$, which can be used to update the value of other position states. For example, we can modify \eqref{eq:nestedset} to obtain a definition of the collection of nested sets as follows:
\begin{align*}
	{\cal X}_{n,d} = {{\cal Z}_n} \cup \left( { \cup _{i = 1}^L{\cal G}_{n + l}^ - } \right),
	{\cal X}_{n,l} = {{\cal Z}_n} \cup {\cal G}^-_{n + 1},
\end{align*}
for all $1\le l \le d-1$.
Replacing ${\cal X}_l$ with  ${\cal X}_{n,l}$ in Algorithm \ref{al:NVI}, NVI can be used to solve TA($n,L$).
Our previous discussions also imply that using the optimal value function of ${\cal C}^{n-1}_L (\rho)$  as the starting point would help NVI converge quickly to the optimal value function of  ${\cal C}^n_L (\rho)$, as the values of root position states are already optimum.

\section{Numerical Results}   \label{sec:exp}
We present numerical results to demonstrate the effectiveness of our methods. Throughout this section, the state information transmissions of all experiments were simulated by Bernoulli trials generated by Matlab on PC. At each time step, the controller of the IOMDP receives the current state information only if the outcome of the Bernoulli trail is 1 (another outcome is 0). The probability of the outcome being 1 in each Bernoulli trial corresponds to the transmission success probability (i.e., SIRP) of the experiment. 

\begin{table}[h]
	\centering
	\begin{tabular}{|ll|c|c|c|}
		\hline
		\multicolumn{2}{|l|}{}                                                       & NVI-1                         & NVI-2                           & VI                        \\ \hline
		\multicolumn{1}{|c|}{}                    & \cellcolor[HTML]{DAE8FC}Time (s) & \cellcolor[HTML]{DAE8FC}11  & \cellcolor[HTML]{DAE8FC}16   & \cellcolor[HTML]{DAE8FC}41   \\ \cline{2-5} 
		\multicolumn{1}{|c|}{\multirow{-2}{*}{\begin{tabular}[c]{@{}c@{}}$|{\cal S}|=40,|{\cal A}|=3$\\ $\rho=0.7,L=6$\end{tabular}} } & Iteration                        & 13                          & 13                           & 49                          \\ \hline
		\multicolumn{1}{|l|}{}                    & \cellcolor[HTML]{DAE8FC}Time (s) & \cellcolor[HTML]{DAE8FC}48  & \cellcolor[HTML]{DAE8FC}80  & \cellcolor[HTML]{DAE8FC}217  \\ \cline{2-5} 
		\multicolumn{1}{|l|}{\multirow{-2}{*}{\begin{tabular}[c]{@{}c@{}}$|{\cal S}|=80,|{\cal A}|=4$\\ $\rho=0.8,L=5$\end{tabular}}} & Iteration                        & 12                          & 16                           & 56                          \\ \hline
		\multicolumn{1}{|l|}{}                    & \cellcolor[HTML]{DAE8FC}Time (s) & \cellcolor[HTML]{DAE8FC}65 & \cellcolor[HTML]{DAE8FC}137  & \cellcolor[HTML]{DAE8FC}356  \\ \cline{2-5} 
		\multicolumn{1}{|l|}{\multirow{-2}{*}{\begin{tabular}[c]{@{}c@{}}$|{\cal S}|=100,|{\cal A}|=5$\\ $\rho=0.9,L=4$\end{tabular}}} & Iteration                        & 9                          & 17                           & 50                          \\ \hline
		\multicolumn{1}{|l|}{}                    & \cellcolor[HTML]{DAE8FC}Time (s) & \cellcolor[HTML]{DAE8FC}375 & \cellcolor[HTML]{DAE8FC}580 & \cellcolor[HTML]{DAE8FC}2051 \\ \cline{2-5} 
		\multicolumn{1}{|l|}{\multirow{-2}{*}{\begin{tabular}[c]{@{}c@{}}$|{\cal S}|=100,|{\cal A}|=5$\\ $\rho=0.9,L=5$\end{tabular}}} & Iteration                        & 9                          & 14                           & 50                          \\ \hline
		\multicolumn{1}{|l|}{}                    & \cellcolor[HTML]{DAE8FC}Time (s) & \cellcolor[HTML]{DAE8FC}84 & \cellcolor[HTML]{DAE8FC}136  & \cellcolor[HTML]{DAE8FC}432  \\ \cline{2-5} 
		\multicolumn{1}{|l|}{\multirow{-2}{*}{\begin{tabular}[c]{@{}c@{}}$|{\cal S}|=200,|{\cal A}|=3$\\ $\rho=0.9,L=5$\end{tabular}}} & Iteration                        & 10                          & 15                           & 54                          \\ \hline
	\end{tabular}
	\caption{Computation time and iterations of different algorithms for solving TA($L$).}
	\label{tab:NVI-perf}
\end{table}

\subsection{Efficiency of NVI}
We conducted a comparison between NVI and standard value iteration (VI) using a set of randomly generated MDPs, as presented in Table \ref{tab:NVI-perf}. Two distinct definitions of $\{\mathcal{X}_i \}$ were examined: NVI-1 adopts the $\{\mathcal{X}_i \}$ definition provided by \eqref{eq:nestedset}, while NVI-2 utilizes the following alternative definition: let $d=L$ and define
\begin{align*}
	\mathcal{X}_l= {{\cal H}_l} =  \cup _{i = 0}^l{{\cal G}_i},\quad l=1,2,\cdots,L.
\end{align*} 
For each experiment in Table \ref{tab:NVI-perf}, the transition matrices and the reward function of the underlying MDP were generated randomly with given state and action spaces. We then constructed the associated TA($L$) and computed its optimal policy using the three algorithms.
In all experiments, NVI-1 and NVI-2 converge to the same value function as VI, leading to the same policy. 
The results clearly demonstrate that both NVI algorithms converge significantly faster than the standard value iteration, with NVI-1 outperforming NVI-2. Additionally, we provide the respective number of iterations required for convergence for each algorithm, with the number of outer iterations counted for NVI-1 and NVI-2. Notably, the number of iterations for both NVIs is substantially lower than that of VI, which aligns with our theoretical analysis. 
It is worth noting that in NVI-2, $d$ is set as $L$ across all groups, whereas in NVI-1, $d$ is randomly selected in each group. Interestingly, we observed that varying the value of $d$ in a reasonable range (e.g., $[L,2L]$) only led to minor fluctuations in performance. This finding demonstrates the robustness of NVI-1 with respect to the parameter $d$.
In summary, NVI is more efficient than VI. In the remaining part of this section, TA($L$) and TA($n,L$) are solved using NVI in all experiments.

\subsection{Case Study: Unmanned Boat with Remote Sensor}
{ We next evaluate the effectiveness of TA($L$) and TA($n,L$) in an IOMDP scenario inspired by a practical application. Specifically, consider an unmanned boat tasked with remaining within a designated target water area to perform monitoring activities. The boat’s movement is prone to failure with a certain probability due to the influence of water waves.
	The boat is permitted to navigate within a predefined operational area, and leaving this area results in the termination of the task. Its position is tracked by a remote sensor, which transmits the location information to the boat at each time step. However, the long distance between the sensor and the boat introduces a communication failure probability of $1-\rho$. To simplify the problem, we discretize the water area into 9 segments, resulting in an MDP with 9 states and 4 actions. Our methods are assessed in this simplified model. Further details about the MDP are provided in Appendix E.
}

As shown in Table \ref{tab:hota}, we consider 4 values of $\rho$. For each $\rho$, we solved the associated TA($L$) using NVI and obtained the TA($L$) policy to control the IOMDP with SIRP $\rho$. We then collected the total discounted reward over a horizon of $10^5$ time steps. The value of each group is the empirical average total discounted reward of $2\times10^4$ independent runs.
The results of TA($L+n$) and TA($n,L$) were obtained by the same procedure. In all experiments, the policies generated by TA($L+n$) and TA($n,L$) are identical; hence their values are equal. We also present the computation time of each method in the table, where the computation time of TA($n,L$) includes the time of solving TA($k,L$) for all $0\le k\le n$.
The results show that TA($L$) with a small $L$ can not generate a satisfactory policy when $\rho$  is small, but we can obtain a better policy by increasing $L$. However, the computation time of TA($L$) increases explosively with $L$. Fortunately, TA($n,L$) can generate the same policy as TA($L+n$) using much less computation time.

\begin{table}[t]
	\centering
	\begin{tabular}{|ll|c|c|c|}
		\hline
		\multicolumn{2}{|l|}{}                                                      & TA($L$)                       & TA($L+n$)                     & TA($n,L$)                     \\ \hline
		\multicolumn{1}{|c|}{}                      & \cellcolor[HTML]{DAE8FC}Value & \cellcolor[HTML]{DAE8FC}366 & \cellcolor[HTML]{DAE8FC}368 & \cellcolor[HTML]{DAE8FC}368 \\ \cline{2-5} 
		\multicolumn{1}{|c|}{\multirow{-2}{*}{$\rho=0.9$}} & Time (s)                      & 0.48                        & 34                          & 0.64                        \\ \hline
		\multicolumn{1}{|l|}{}                      & \cellcolor[HTML]{DAE8FC}Value & \cellcolor[HTML]{DAE8FC}310 & \cellcolor[HTML]{DAE8FC}318 & \cellcolor[HTML]{DAE8FC}318 \\ \cline{2-5} 
		\multicolumn{1}{|l|}{\multirow{-2}{*}{$\rho=0.8$}} & Time (s)                      & 0.52                        & 34                          & 0.65                        \\ \hline
		\multicolumn{1}{|l|}{}                      & \cellcolor[HTML]{DAE8FC}Value & \cellcolor[HTML]{DAE8FC}196 & \cellcolor[HTML]{DAE8FC}215 & \cellcolor[HTML]{DAE8FC}215 \\ \cline{2-5} 
		\multicolumn{1}{|l|}{\multirow{-2}{*}{$\rho=0.6$}} & Time (s)                      & 0.45                        & 30                          & 0.72                        \\ \hline
		\multicolumn{1}{|l|}{}                      & \cellcolor[HTML]{DAE8FC}Value & \cellcolor[HTML]{DAE8FC}155 & \cellcolor[HTML]{DAE8FC}175 & \cellcolor[HTML]{DAE8FC}175 \\ \cline{2-5} 
		\multicolumn{1}{|l|}{\multirow{-2}{*}{$\rho=0.5$}} & Time (s)                      & 0.44                        & 30                          & 0.80                        \\ \hline
	\end{tabular}
	\caption{Performance of TA($L$) and TA($n,L$), $L=2,n=4$.}
	\label{tab:hota}
\end{table}

\begin{figure}[!t]
	\centering
	\begin{subfigure}[b]{0.4\linewidth}
		\includegraphics[width=\linewidth]{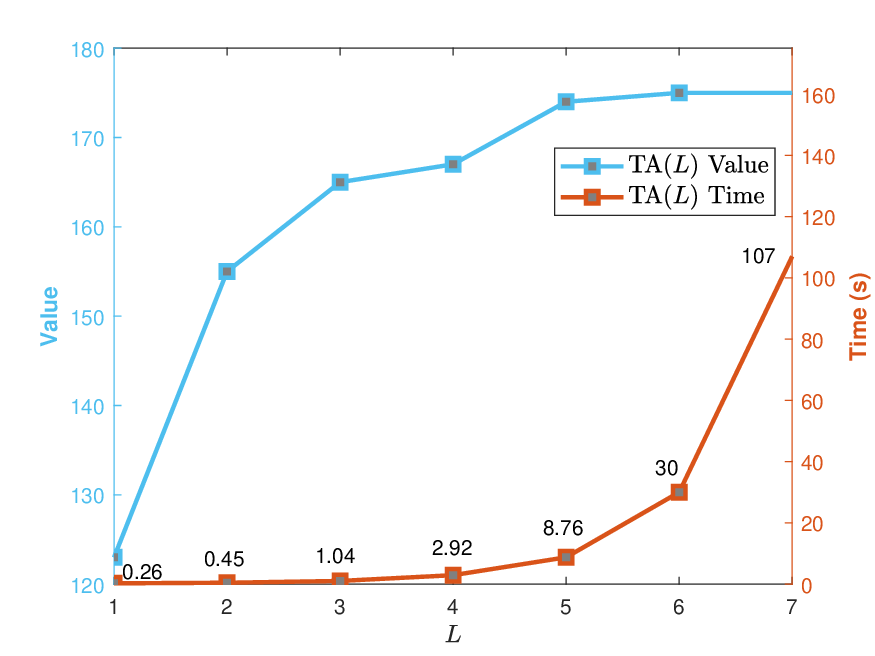}
		\caption{TA($L$) performance.}
		\label{fig:expTAL}
	\end{subfigure}
	\begin{subfigure}[b]{0.4\linewidth}
		\includegraphics[width=\linewidth]{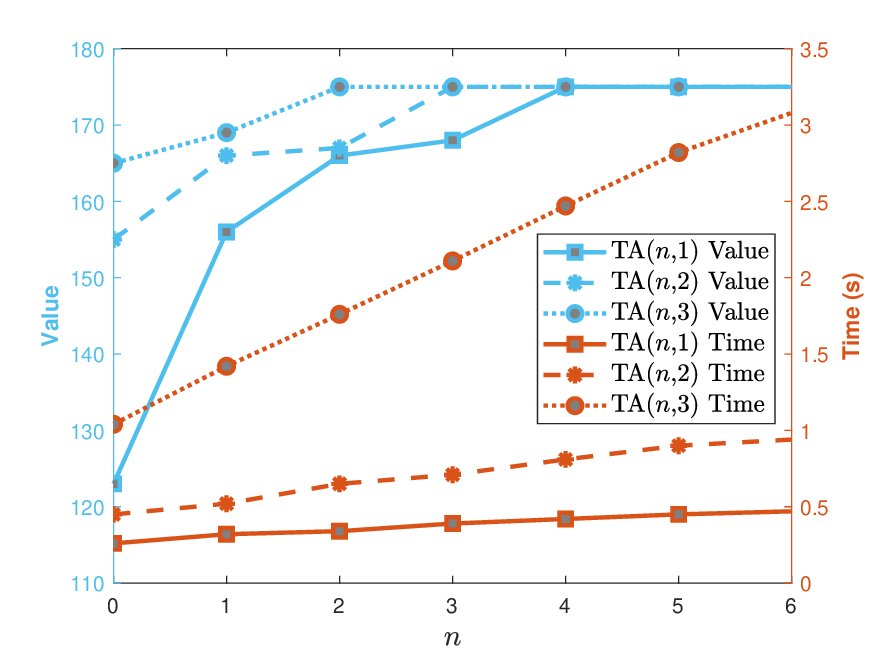}
		\caption{TA($n,L$) performance.}
		\label{fig:expHota}
	\end{subfigure}
	
	\caption{Values and computation time of TA($L$) and TA($n,L$) policies.}
	\label{fig:exp}
\end{figure}

We carried out more experiments for the last group ($\rho = 0.5$) of Table \ref{tab:hota} to show how $L$ and $n$ affect the performance of TA($L$) and TA($n,L$). Fig.~\ref{fig:expTAL} shows that the value of the TA($L$) policy increases with $L$ and eventually becomes stable after $L\ge 6$. This verifies our analysis in Section \ref{sec:HOTA}---the larger the value of $L$, the TA($L$) policy is optimal for more position states, hence the closer the TA($L$) policy is to the optimal policy. Given an initial position state in ${\cal G}_0$, the value difference between the TA($L$) policy and the optimal policy diminishes quickly as $L$ increases. The disadvantage of TA($L$) is also obvious: the computation time increases exponentially with $L$. In Fig.~\ref{fig:expHota}, we fix the value of $L$ and examine how the parameter $n$ affects the value and the computation time of the TA($n,L$) policy. Three different $L$ are compared. TA($n,L$) shows excellent performance in this experiment---it achieves the same value as TA($L+n$) with much less computation time. It is worth noting that the computation time of TA($n,L$) only increases linearly with $n$, demonstrating excellent scalability.

\begin{table}[t]
	\centering
	\begin{tabular}{|l|c|c|c|}
		\hline
		& TA(2) & TA(3) & TA(6) \\ \hline
		$|{\cal S}|=20,|{\cal A}|=4$ & 69    & 69    & 69    \\ \hline
		$|{\cal S}|=30,|{\cal A}|=3$ & 90    & 90    & 90    \\ \hline
		$|{\cal S}|=40,|{\cal A}|=4$ & 192   & 193   & 193   \\ \hline
		$|{\cal S}|=50,|{\cal A}|=5$ & 395   & 397   & 397   \\ \hline
	\end{tabular}
	\caption{Values of TA($L$) policies ($\rho=0.9$).}
	\label{tab:tal}
\end{table}

The above experiments show that, compared with TA($n,L$), TA($L$) is much more inefficient and unscalable. 
However, we can show that TA($L$) with small  $L$ is good enough in some scenarios. The experiments were conducted in IOMDPs generated randomly as before. We set the SIRP to be $\rho=0.9$ and the initial position state belong to ${\cal G}_0$ in all these experiments. Table \ref{tab:tal} shows the values of multiple TA($L$) policies, where the value of each group is obtained by the same procedure as in Table \ref{tab:hota}. 
In these experiments, the performance of TA(2) is (nearly) the same as that of TA(6), implying that TA(2) is already near-optimal. Considering that TA($L$) is easier to implement than TA($n,L$), it may be a good choice to use TA($L$) when a small $L$  is good enough.

\section{Conclusion}  \label{sec: conl}
This paper studied intermittently observable MDPs. We assumed that the transmissions of state information follow a Bernoulli lossy process. The problem of selecting actions for the MDP in the presence of state information losses has been investigated systematically. Two methods of formulating the IOMDP were used for different purposes. We first formulated the problem as a belief MDP to analyze the effect of state information losses. We proved that the expected total discounted reward is a continuous and increasing function of the state information reception probability. An upper bound for the performance regret caused by state information losses is derived.
Then, we reformulated the IOMDP as a tree MDP to develop finite-state approximations, TA($L$) and TA($n,L$), for the problem. We showed analytically that TA($L$) could well approximate the optimal value function for part of the states, thus allowing it to generate a near-optimal policy. Inspired by the approximation error bound of TA($L$), we further developed TA($n,L$), which is more efficient than TA($L$) because it excludes the redundant states when computing policies.
In addition, we also proposed a nested value iteration algorithm for TA($L$) and TA($n,L$). Convergence and computational complexity analysis is provided for interpreting the efficiency of the algorithm. Finally, we validated the effectiveness of the proposed methods by numerical results.


\section*{Appendix A}    \label{app:sec3}


\subsection*{A.1 Proof of Theorem \ref{theo:phi_b}}
In this part we prove Theorem \ref{theo:phi_b} from Section~\ref{sec:Structural}. For notation simplicity, $T_\rho ^\pi ({\bf{x}},{\bf{y}})$ will stand for the probability of the belief state transitioning from ${\bf{x}} \in {\cal W}$  to  ${\bf{y}} \in {\cal W}$, i.e., 
\begin{align*}
	T_\rho ^\pi ({\bf{x}},{\bf{y}}) = \Pr \left( {{{\bf{b}}_t} = {\bf{y}}|{{\bf{b}}_{t - 1}} = {\bf{x}},{a_{t - 1}} = \pi ({\bf{x}})} \right).
\end{align*}

\begin{proof}[Theorem \ref{theo:phi_b}]
	Consider $0<\rho \le v \le 1$. Suppose that $\pi_\rho \in \Pi$ is an optimal policy for ${\cal B}(\rho)$. Let $T_\rho ^{{\pi _\rho }}$ denote the transition matrix of ${\cal B}(\rho)$ under policy $\pi_\rho$. Then
	\begin{align} \label{eq:thm2-1}
		\phi [\rho ] = {\phi ^{{\pi _\rho }}}[\rho ] = {R_{{\pi _\rho }}} + \beta T_\rho ^{{\pi _\rho }}{\phi ^{{\pi _\rho }}}[\rho ].
	\end{align}
	Note that policy $\pi_\rho$ can also be applied to the belief MDP ${\cal B}(v)$. We have
	\begin{align} \label{eq:thm2-2}
		\phi [v] \ge {\phi ^{{\pi _\rho }}}[v] = {R_{{\pi _\rho }}} + \beta T_v^{{\pi _\rho }}{\phi ^{{\pi _\rho }}}[v].
	\end{align}
	From \eqref{eq:thm2-1} and \eqref{eq:thm2-2},
	\begin{align}
		&{\phi ^{{\pi _\rho }}}[v] - {\phi ^{{\pi _\rho }}}[\rho ] = \notag \\
		&\beta T_v^{{\pi _\rho }}\left( {{\phi ^{{\pi _\rho }}}[v] - {\phi ^{{\pi _\rho }}}[\rho ]} \right) + \beta \left( {T_v^{{\pi _\rho }} - T_\rho ^{{\pi _\rho }}} \right){\phi ^{{\pi _\rho }}}[\rho ].
	\end{align}
	It follows that
	\begin{align} \label{eq:thm2-diff}
		{\phi ^{{\pi _\rho }}}[v] - {\phi ^{{\pi _\rho }}}[\rho ] = \beta {\left( {I - \beta T_v^{{\pi _\rho }}} \right)^{ - 1}}\left( {T_v^{{\pi _\rho }} - T_\rho ^{{\pi _\rho }}} \right){\phi ^{{\pi _\rho }}}[\rho ].
	\end{align}
	Note that ${\left( {I - \beta T_v^{{\pi _\rho }}} \right)^{ - 1}}$ exists and is non-negative (i.e., all entries are non-negative) because $T_v^{{\pi _\rho }}$  is a stochastic matrix. Specifically,
	\begin{align*}
		{\left( {I - \beta T_v^{{\pi _\rho }}} \right)^{ - 1}} = \sum\limits_{n = 0}^\infty  {{\beta ^n}{{(T_v^{{\pi _\rho }})}^n}}.
	\end{align*}
	Let $\eta  = \left( {T_v^{{\pi _\rho }} - T_\rho ^{{\pi _\rho }}} \right){\phi ^{{\pi _\rho }}}[\rho ]$. Just like ${\phi ^{{\pi _\rho }}}( \cdot ,\rho )$, $\eta$ can be viewed as a function on set ${\cal W}$. For an arbitrary $\mathbf{b}\in {\cal W}$, let $\tau  = P_{{\pi _\rho }({\bf{b}})}^\top{\bf{b}}$. Then we can express $\eta$ in component notation as
	\begin{align*}
		\eta ({\bf{b}}) =& \left( {T_v^{{\pi _\rho }}({\bf{b}},\tau ) - T_\rho ^{{\pi _\rho }}({\bf{b}},\tau )} \right){\phi ^{{\pi _\rho }}}(\tau ,\rho ) \\
		&+ \sum\limits_{i \in {\cal S}} {\left( {T_v^{{\pi _\rho }}({\bf{b}},{{\bf{e}}_i}) - T_\rho ^{{\pi _\rho }}({\bf{b}},{{\bf{e}}_i})} \right){\phi ^{{\pi _\rho }}}({{\bf{e}}_i},\rho )} \\
		=& (v - \rho )\left[ {\sum\limits_{i \in {\cal S}} {\tau (i){\phi ^{{\pi _\rho }}}({{\bf{e}}_i},\rho ) - {\phi ^{{\pi _\rho }}}(\tau ,\rho )} } \right].
	\end{align*}
	Note that $\tau  = \sum\limits_{i \in {\cal S}} {\tau (i){{\bf{e}}_i}} $ and that ${\phi ^{{\pi _\rho }}}({\bf{x}},\rho ) = \phi ({\bf{x}},\rho )$ is a convex function of ${\bf x} \in {\cal W}$. It follows that
	\begin{align*}
		\sum\limits_{i \in {\cal S}} {\tau (i){\phi ^{{\pi _\rho }}}({{\bf{e}}_i},\rho ) - {\phi ^{{\pi _\rho }}}(\tau ,\rho )}  \ge 0.
	\end{align*}
	Therefore, $\eta ({\bf{b}}) \ge 0$ for any ${\bf{b}} \in {\cal W}$. The product of a non-negative matrix and a non-negative vector is clearly non-negative. That is,
	\begin{align}
		{\phi ^{{\pi _\rho }}}[v] - {\phi ^{{\pi _\rho }}}[\rho ] = \beta {\left( {I - \beta T_v^{{\pi _\rho }}} \right)^{ - 1}}\eta  \ge 0.
	\end{align}
	It follows immediately that $\phi [v] \ge {\phi ^{{\pi _\rho }}}[v] \ge {\phi ^{{\pi _\rho }}}[\rho ] = \phi [\rho ]$. This completes the proof.

\end{proof}

\subsection*{A.2 Proof of Theorem \ref{thm:regretbound}}
This part proves Theorem \ref{thm:regretbound} from Section~\ref{sec:Structural}.

\begin{proof}
	Consider an arbitrary $\rho\in (0,1)$. Let $v\in [\rho,1]$ and $\pi_v$ denote an optimal policy for ${\cal B}(v)$. It follows from Theorem \ref{theo:phi_b} that ${\phi ^{{\pi _v}}}[v] = \phi [v] \ge \phi [\rho ] \ge {\phi ^{{\pi _v}}}[\rho ]$. Therefore,
	\begin{align}
		&{\left\| {\phi [v] - \phi [\rho ]} \right\|_\infty } =\notag \\
		& \mathop {\max }\limits_{{\bf{b}} \in {\cal W}} \left| {\phi ({\bf{b}},v) - \phi ({\bf{b}},\rho )} \right| \le {\lVert {{\phi ^{{\pi _v}}}[v] - {\phi ^{{\pi _v}}}[\rho ]} \rVert_\infty }
	\end{align}
	where $||\cdot||_\infty$ denotes the max norm. Using a similar argument as in \eqref{eq:thm2-diff} yields
	\begin{align} \label{eq:thm5-diff}
		{\phi ^{{\pi _v}}}[v] - {\phi ^{{\pi _v}}}[\rho ] = \beta {\left( {I - \beta T_\rho ^{{\pi _v}}} \right)^{ - 1}}\left( {T_v^{{\pi _v}} - T_\rho ^{{\pi _v}}} \right)\phi [v].
	\end{align}
	Let $\eta  = \left( {T_v^{{\pi _v}} - T_\rho ^{{\pi _v}}} \right)\phi [v]$. Then for any ${\bf b}\in {\cal W}$, let $\tau  = P_{{\pi _v}({\bf{b}})}^\top{\bf{b}}$, we have
	\begin{align} \label{eq:thm5-eta}
		\eta ({\bf{b}}) = \beta (v - \rho )\left[ {\sum\limits_{i \in {\cal S}} {\tau (i)\phi ({{\bf{e}}_i},v) - \phi (\tau ,v)} } \right].
	\end{align}
	Since the reward function is bounded, the value function $\phi(\cdot,v)$ is also bounded. The convexity of $\phi(\cdot,v)$  implies that the term within the square brackets is non-negative. As a result,
	\begin{align*}
		{\sigma _v} \buildrel \Delta \over = \mathop {\max }\limits_{\tau  \in {\cal W}} \left\{ {\sum\limits_{i \in {\cal S}} {\tau (i)\phi ({{\bf{e}}_i},v) - \phi (\tau ,v)} } \right\}
	\end{align*}
	is non-negative and bounded. It follows that
	\begin{align}
		0 \le {\phi ^{{\pi _v}}}[v] - {\phi ^{{\pi _v}}}[\rho ] &\le \beta (v - \rho ){\left( {I - \beta T_v^{{\pi _v}}} \right)^{ - 1}}{\sigma _v}{\bf{e}} \notag \\
		&= (v - \rho )\frac{{\beta {\sigma _v}}}{{1 - \beta }}{\bf{e}},
	\end{align}
	where ${\bf e}$ is the vector with all elements being 1. Clearly, for any $\epsilon>0$,
	\begin{align*}
		\left| {v - \rho } \right| < \epsilon \left( {1 - \beta } \right)/\beta {\sigma _v}\  \to \ \mathop {\max }\limits_{{\bf{b}} \in {\cal W}} \left| {\phi ({\bf{b}},v) - \phi ({\bf{b}},\rho )} \right| \le \epsilon. 
	\end{align*}
	Using a similar argument could show that the above conclusion holds for $v\le \rho$. This proves the continuity of $\phi({\bf b},\rho)$.
	
	The bound for $\phi[1]-\phi[\rho]$ is obtained by substituting $v=1$ into \eqref{eq:thm5-diff} and \eqref{eq:thm5-eta}. Using Lemma \ref{lem:rho=1} yields the expression of  $\eta$ presented in Theorem \ref{thm:regretbound}.
	
\end{proof}

\section*{Appendix B}
\subsection*{B.1 Proof of Lemma \ref{lem:2formulations}}
\begin{proof}
	For any $h\in  {\cal H}$ and $a\in {\cal A}$, let ${\bf b}=g(h)$, $\tau_a=(h,a)$, and ${\bf b}_a = g(\tau_a)$. Note that
	\begin{align*}
		P_a^\top g(h) = P_a^\top {\bf{b}} = g({\tau _a}) = {{\bf{b}}_a}.
	\end{align*}
	Then by definition, we have
	\begin{align*}
		&D(i|h,a) = \rho {{\bf{b}}_a}(i) = T({{\bf{e}}_i}|{\bf{b}},a),\quad \forall i \in {\cal S}\\
		&D({\tau _a}|h,a) = 1 - \rho  = T({{\bf{b}}_a}|{\bf{b}},a).
	\end{align*}
	The Bellman equation for ${\cal C}(\rho)$ can be written as
	\begin{align} \label{eq:lem6-c}
		\varphi (h)= &\mathop {\max }\limits_{a \in {\cal A}} \{ R\left( {g(h),a} \right) + \beta \sum\limits_{i \in {\cal S}} T({{\bf{e}}_i}|{\bf{b}},a)\varphi (i) \notag \\
		&+ \beta T({{\bf{b}}_a}|{\bf{b}},a)\varphi ({\tau _a})  \}.
	\end{align}
	Note also that the Bellman equation for ${\cal B}(\rho)$ is
	\begin{align} \label{eq:lem6-b}
		\phi ({\bf{b}}) =& \mathop {\max }\limits_{a \in {\cal A}} \{ R\left( {{\bf{b}},a} \right) + \beta \sum\limits_{i \in {\cal S}} T({{\bf{e}}_i}|{\bf{b}},a)\phi ({{\bf{e}}_i}) \notag \\
		&+ \beta T({{\bf{b}}_a}|{\bf{b}},a)\phi ({{\bf{b}}_a})  \}.
	\end{align}
	Clearly, \eqref{eq:lem6-c} and \eqref{eq:lem6-b} are of the same form for any ${\bf b}=g(h)$. Since the Bellman equation has a unique solution, we conclude that $\phi({\bf b}) = \varphi({h})$ if ${\bf b}=g(h)$. This proves statement 1. Statement 2 follows immediately: for any $h,h'\in {\cal H}$ satisfying $g(h) = g(h')$, we have $\varphi (h') = \phi ({\bf{b}}) = \varphi (h)$. In addition, statement 1 implies that $\phi ({{\bf{e}}_i}) = \varphi (i)$ for all $i\in {\cal S}$ and $\phi ({{\bf{b}}_a}) = \varphi ({\tau _a})$ for any $a\in {\cal A}$. Then statement 3 can be verified using \eqref{eq:lem6-c} and \eqref{eq:lem6-b}.
	
\end{proof}

\subsection*{B.2 Proof of Theorem \ref{thm:TAL-bound}}
\begin{proof}
	We first construct an auxiliary MDP, denoted by ${\cal C}'(\rho ) = ({\cal H},{\cal A},\rho ,M,\bar R,\beta )$ to facilitate the analysis of the approximation error of TA($L$). In particular, the belief state space ${\cal H}$ , action space ${\cal A}$, SIRP  $\rho$, reward function  ${\bar R}$, and discount factor $\beta$  are identical to those of ${\cal C}(\rho)$. The transition kernel  $M$ is defined as follows:
	\begin{itemize}
		\item [(1)] For any $h\in {\cal G}_n$ with $n\neq L$, $M(h'|h,a) = D(h'|h,a)$, where $h'\in {\cal H}$ and $a\in {\cal A}$;
		\item [(2)] For any $h\in {\cal G}_L$, let $\tau_a = (h,a)$ and ${\bf b}_a = g(\tau_a)$, where $a\in {\cal A}$. Then 
		\begin{align*} 
			M(h'|h,a)  = \Pr \left( {h'|h,a} \right) = \begin{cases}
				1 - \rho ,& {\text{ if  }}h' =h\in {\cal G}_L\\
				\rho {\bf b}_a(i),& {\text{ if  }}h' = i \in {\cal S}\\
				0,&{\text{    otherwise.}}
			\end{cases}
		\end{align*}
	\end{itemize}
	
	On the one hand, since ${\cal C}(\rho)$ and ${\cal C}'(\rho)$ have the same state and action spaces, they share the same set of deterministic stationary policies, say $\Pi_h$. On the other hand, note that
	${\cal H} = {{\cal H}_L} \cup \left( { \cup _{n = L + 1}^\infty {{\cal G}_n}} \right)$ and that the probability of the position state transitioning from ${\cal H}_L$ to $\left( { \cup _{n = L + 1}^\infty {{\cal G}_n}} \right)$ is 0. If the initial position state is $h_0\in {\cal H}_L$, then ${\cal C}'(\rho)$ reduces to ${\cal C}_L(\rho)$. Denote by $\psi^\pi(\cdot)$ the value function of ${\cal C}'(\rho)$ under policy $\pi\in \Pi_h$ and $\psi(\cdot)$ the optimal value function of ${\cal C}'(\rho)$. Then
	\begin{align*}
		\psi (h) = {\varphi _L}(h),\quad \forall h \in {{\cal H}_L}.
	\end{align*}
	As we will show below, it is more convenient to compare two MDPs with the same state and action spaces. Hence ${\cal C}'(\rho)$  will serve as a bridge that connects ${\cal C}(\rho)$  and  ${\cal C}_L(\rho)$ so that we can characterize the approximation error by comparing $\varphi$ with $\psi$.
	
	For any $\pi \in \Pi_h$, using the same argument as in \eqref{eq:thm2-diff} yields
	\begin{align} \label{eq:thm8-diff}
		{\varphi ^\pi } - {\psi ^\pi } = \beta {\left( {I - \beta {M^\pi }} \right)^{ - 1}}\left( {{D^\pi } - {M^\pi }} \right){\varphi ^\pi }.
	\end{align}
	Let ${\xi ^\pi } = \left( {{D^\pi } - {M^\pi }} \right){\varphi ^\pi }$ and $h'=(h,\pi(h))$. Then ${\xi ^\pi }: {\cal H}\to \mathbb{R}$ is given by
	\begin{align}  \label{eq:8xi}
		{\xi ^\pi }(h) = \begin{cases}
			\beta (1 - \rho )\left[ {{\varphi ^\pi }(h') - {\varphi ^\pi }(h)} \right],&{\text{if }}h \in {{\cal G}_L}\\
			0,&{\text{otherwise.}}
		\end{cases}
	\end{align}
	The RHS of \eqref{eq:thm8-diff} can be viewed as the expected total discounted reward of a Markov reward process whose reward function is ${\xi ^\pi }$  and the Markov chain is governed by the transition kernel $M^\pi$. Denote by $\{x_t:t = 0,1,2,\cdots\}$ the Markov reward process. We then can write $\varphi^\pi - \psi^\pi$ in component form as, $\forall h \in {\cal H}$,
	\begin{align*} 
		{\varphi ^\pi }(h) - {\psi ^\pi }(h) = \sum\limits_{t = 0}^\infty  {{\beta ^t}\sum\limits_{y \in {\cal H}} {\Pr \left( {{x_t} = y|{x_0} = h} \right){\xi ^\pi }(y)} }.
	\end{align*}
	Substituting \eqref{eq:8xi} into the above formula yields
	\begin{align*}
		\left| {{\varphi ^\pi }(h) - {\psi ^\pi }(h)} \right| \le {\left\| {{\xi ^\pi }} \right\|_\infty }\sum\limits_{t = 0}^\infty  {{\beta ^t}\sum\limits_{y \in {{\cal G}_L}} {\Pr \left( {{x_t} = y|{x_0} = h} \right)} }. 
	\end{align*}
	For any $h\in {\cal G}_k$, it can be proved that
	\begin{align} \label{eq:8-OL}
		&\sum\limits_{t = 0}^\infty  {{\beta ^t}\sum\limits_{y \in {{\cal G}_L}} {\Pr \left( {{x_t} = y|{x_0} = h} \right)} } \notag \\
		=& \sum\limits_{t = L - k}^{L - 1} {{\beta ^t}{{(1 - \rho )}^t}}  + \frac{{{\beta ^L}{{(1 - \rho )}^L}}}{{1 - \beta }}: = O_L(k)
	\end{align} 
	{
		To see this, define
		\begin{align*}
			{p_n}(t) = \sum\limits_{y \in {{\cal G}_n}} {\Pr \left( {{x_t} = y|{x_0} = h} \right)}  = \Pr \left( {{x_t} \in {{\cal G}_n}|{x_0} = h} \right),
		\end{align*}
		where $n=0,1,2,\cdots,L$. Since $\{x_t:t=0,1,2,\cdots\}$ is governed by the transition kernel $M^\pi$, we have
		\begin{align} \label{eq:eL}
			{p_L}(t + 1) &= (1 - \rho )\left[ {{p_L}(t) + {p_{L - 1}}(t)} \right],\\ \label{eq:e0}
			{p_0}(t + 1) &= \rho \sum\limits_{n = 0}^L {{p_n}(t)}  = \rho ,\\ \label{eq:en}
			{p_n}(t + 1) &= (1 - \rho ){p_{n - 1}}(t),\quad n = 1, \cdots ,L - 1.
		\end{align}
		From \eqref{eq:e0} and \eqref{eq:en}, for $n\in \{1,2,\cdots,L-1 \}$ and $t\ge n$,
		\begin{align}
			{p_n}(t + 1) = {(1 - \rho )^n}{p_0}(t - n + 1) = \rho {(1 - \rho )^n}.
		\end{align}
		It follows that, for $t\ge L$,
		\begin{align*}
			{p_L}(t) = {(1 - \rho )^{t - L}}{p_L}(L) + {(1 - \rho )^L}\left[ {1 - {{(1 - \rho )}^{t - L}}} \right].
		\end{align*}
		Then we can derive 
		\begin{align}  \label{eq:e-sum}
			\sum\limits_{t = L}^\infty  {{\beta ^t}{p_L}(t)} = \frac{{{\beta ^L}\left[ {{p_L}(L) - {{(1 - \rho )}^L}} \right]}}{{1 - \beta (1 - \rho )}} + \frac{{{\beta ^L}{{(1 - \rho )}^L}}}{{1 - \beta }}.
		\end{align}
		Given that $x_0=h\in {\cal G}_k$, it is easy to verify that
		\begin{align}  \label{eq:e-pl}
			{p_L}(t) =  \begin{cases}
				0,&{\text{ if }}t < L - k\\
				{(1 - \rho )^t},&{\text{ if }}L - k \le t \le L
			\end{cases} 
		\end{align}
		Since $p_L(L) = (1-\rho)^L$, the first term in \eqref{eq:e-sum} is 0. Equation \eqref{eq:8-OL} follows immediately from \eqref{eq:e-sum} and \eqref{eq:e-pl}.
		
	}
	
	Note that $O_L(k)$ is independent of policy $\pi$. Since $\varphi^\pi\ge 0$ for all $\pi$ and $\rho$, we have
	\begin{align*}
		{\left\| {{\xi ^\pi }} \right\|_\infty } &= \mathop {\max }\limits_{h \in {{\cal G}_L}} \beta (1 - \rho )\left| {{\varphi ^\pi }(h') - {\varphi ^\pi }(h)} \right| \\
		& \le \beta (1 - \rho )\mathop {\max }\limits_{h \in {{\cal G}_L} \cup {{\cal G}_{L + 1}}} {\varphi ^\pi }(h)
	\end{align*}
	The inequality follows from the fact that $h\in {\cal G}_L$	and $h'\in {\cal G}_{L+1}$. Define $	{\delta _L} =  \mathop{\max }\limits_{h \in {{\cal G}_L} \cup {{\cal G}_{L + 1}}} \varphi (h).$
	Then ${\left\| {{\xi ^\pi }} \right\|_\infty } \le \beta (1 - \rho ){\delta _L}$ for all $\pi$. Therefore, for any $\pi\in \Pi_h$,
	\begin{align} \label{eq:8-absdiff}
		\left| {{\varphi ^\pi }(h) - {\psi ^\pi }(h)} \right| \le \beta (1 - \rho ){\delta _L}{O_L}(k),\quad \forall h \in {{\cal G}_k}.
	\end{align}
	Now, denote by $\pi$ and $\pi'$ the optimal policies for ${\cal C}(\rho)$ and ${\cal C}'(\rho)$, respectively. From \eqref{eq:8-absdiff}, for any $h\in {\cal G}_k,k=0,1,\cdots,L$,
	\begin{align} \label{eq:8left}
		\varphi (h) = {\varphi ^\pi }(h) &\le {\psi ^\pi }(h) + \beta (1 - \rho ){\delta _L}{O_L}(k) \notag \\
		& \le \psi (h) + \beta (1 - \rho ){\delta _L}{O_L}(k),\\ 
		\psi (h) = {\psi ^{\pi '}}(h) &\le {\varphi ^{\pi '}}(h) + \beta (1 - \rho ){\delta _L}{O_L}(k) \notag \\ \label{eq:8right}
		&\le \varphi (h) + \beta (1 - \rho ){\delta _L}{O_L}(k).
	\end{align}
	Putting together \eqref{eq:8left} and \eqref{eq:8right} yields
	\begin{align}
		\left| {\varphi (h) - \psi (h)} \right| \le \beta (1 - \rho ){\delta _L}{O_L}(k),\quad \forall h \in {{\cal G}_k}.
	\end{align}
	Since $\psi(h)=\varphi_L(h)$ for all $h\in {\cal H}_L$, the desired result follows immediately.
	
\end{proof}

\section*{Appendix C}

\subsection*{C.1 Proof of Lemma \ref{lem:CLn}}
\begin{proof}
	Let $\mu_{L+n}$ denote the optimal policy for ${\cal C}_{L+n}(\rho)$. According to our construction, ${\cal C}^n_L(\rho)$ and ${\cal C}_{L+n}(\rho)$ having the same optimal action for every $h \in  \cup _{i = 0}^{n - 1}{{\cal Z}_i}$ means that
	\begin{align*}
		{{\cal Z}_i}  =  \begin{cases}
			{{\cal G}_0},& i = 0,\\
			\left\{ {\left( {h,\mu _{L+n}(h)} \right):h \in {{\cal Z}_{i - 1}}} \right\},& 1 \le i \le n,
		\end{cases}
	\end{align*}
	In addition, each $h \in  \cup _{i = 0}^{n - 1}{{\cal Z}_i}$ in MDP ${\cal C}^n_L(\rho)$ has only one available action, i.e., $\mu_{L+n}(h)$.
	
	The quantity $\varphi_{L+n}(h)$ is the expected total discounted reward generated by ${\cal C}_{L+n}(\rho)$ with policy $\mu_{L+n}$ and initial position state $h$. We thus can express $\varphi_{L+n}(h)$ as
	\begin{align} \label{eq:Lem12-1}
		{\varphi _{L + n}}(h) = &\sum\limits_{t = 0}^\infty  {{\beta ^t}} \sum\limits_{x \in {{\cal H}_{L + n}}} [ \bar R\left( x,\mu _{L + n}(x) \right) \times \notag \\
		&\Pr \left( {{h_t} = x,{\mu _{L + n}}(x)|{h_0} = h,{\mu _{L + n}}} \right)].
	\end{align}
	For ${\cal C}_{L+n}(\rho)$ with initial position state $h$ and being governed by policy $\mu_{L+n}$, ${{\cal H}_{L + n}} - {\cal H}_L^n$ is a set of redundant position states. Therefore, for any $h\in {\cal H}^n_L$ and $x\notin {\cal H}^n_L$,
	\begin{align} \label{eq:Lem12-2}
		\Pr \left( {{h_t} = x,{\mu _{L + n}}(x)|{h_0} = h,{\mu _{L + n}}} \right) = 0, \quad \forall t.
	\end{align}
	Denote by $\kappa$ a policy for ${\cal C}^n_L(\rho)$ with $\kappa (h) = {\mu _{L + n}}(h)$ for any $h \in {\cal H}_L^n$. Then \eqref{eq:Lem12-1} and \eqref{eq:Lem12-2} implies that
	\begin{align*} 
		&{\varphi _{L + n}}(h) = \sum\limits_{t = 0}^\infty  {{\beta ^t}} \sum\limits_{x \in {\cal H}_L^n} [ \bar R\left( {x,{\mu _{L + n}}(x)} \right) \times \notag \\
		& \qquad \qquad \quad \Pr \left( {{h_t} = x,{\mu _{L + n}}(x)|{h_0} = h,{\mu _{L + n}}} \right)]  \\ 
		=& \sum\limits_{t = 0}^\infty  {{\beta ^t}} \sum\limits_{x \in {\cal H}_L^n} {\Pr \left( {{h_t} = x,\kappa (x)|{h_0} = h,\kappa } \right)\bar R\left( {x,\kappa (x)} \right)}. 
	\end{align*}
	The last line of the above formula is the expected total discounted reward generated by ${\cal C}^n_L(\rho)$ with policy $\kappa$ and initial position state $h$. It follows that ${\varphi _{L + n}}(h) \le \varphi _L^n(h)$ for all $h\in {\cal H}^n_L$.
	
	We next show ${\varphi _{L + n}}(h) \ge \varphi _L^n(h)$ for all $h\in {\cal H}^n_L$. Denote by $\Pi^c$ the set of all deterministic policies for ${\cal C}_{L+n}(\rho)$. Given the optimal policy $\mu_{L+n}$, define $\Pi _n^c = \{ \kappa  \in {\Pi ^c}:\kappa (h) = {\mu _{L + n}}(h),\forall h \in  \cup _{i = 0}^{n - 1}{{\cal Z}_i}\} $. Then for any $h\in {\cal H}^n_L$,
	\begin{align*}
		&\varphi _L^n(h) \\
		&= \mathop {\max }\limits_{\kappa  \in \Pi _n^c} \sum\limits_{t = 0}^\infty  {{\beta ^t}} \sum\limits_{x \in {\cal H}_L^n} {\Pr \left\{ {{h_t} = x,\kappa (x)|{h_0} = h,\kappa } \right\}\bar R\left( {x,\kappa (x)} \right)} \\
		&\le \mathop {\max }\limits_{\kappa  \in {\Pi ^c}} \sum\limits_{t = 0}^\infty  {{\beta ^t}} \sum\limits_{x \in {\cal H}_L^n} {\Pr \left\{ {{h_t} = x,\kappa (x)|{h_0} = h,\kappa } \right\}\bar R\left( {x,\kappa (x)} \right)}  \\
		&= {\varphi _{L + n}}(h). 
	\end{align*}
	The desired result follows immediately.
	
\end{proof}

\subsection*{C.2 Proof of Theorem \ref{thm:policies}}
\begin{proof}
	For $k\in \{0,1,\cdots,n \}$, define
	\begin{align*}
		O_L(k) = \frac{{{\beta ^{L + k + 1}}{{(1 - \rho )}^{L + k + 1}}}}{{1 - \beta }} + \sum\limits_{i = 0}^k {{\beta ^{L + k - i}}{{(1 - \rho )}^{L + k - i}}}.
	\end{align*}
	Then according to Theorem \ref{thm:TAL-bound}, the approximation error of $\varphi_{L+k}(h)$ for $h\in {\cal G}_{k+1}$ is bounded by
	\begin{align} \label{eq:thm13-1}
		&\left| {{\varphi _{L + k}}(h) - \varphi (h)} \right| \le {\delta _{L + k}}O_L(k) \le \delta (n,L)O_L(k) \notag \\
		\le& \delta (n,L)O_L(0) \le \epsilon(n),\quad \forall h \in {{\cal G}_{k + 1}},
	\end{align} 
	where $k\in \{0,1,\cdots,n \}$. The penultimate inequality follows from the fact that $O_L(k)$ decreases with $k$:
	\begin{align*}
		O_L(k + 1) - O_L(k) 
		= \frac{{{\beta ^{L + k + 2}}{{(1 - \rho )}^{L + k + 1}}}}{{1 - \beta }}(1 - \rho  - 1) < 0.
	\end{align*}
	Denote by $\mu_{L+k}$ the optimal policy for ${\cal C}_{L+k}(\rho)$. Since the error bound increases monotonically with the layer index, \eqref{eq:thm13-1} also holds for any $h\in {\cal G}_i,i=0,1,\cdots,k$. Then according to Fact 1, we have
	\begin{itemize}
		\item [(a)] $\mu_{L+k}(h)$ is an optimal action for $h \in  \cup _{i = 0}^k{{\cal G}_i}$ in ${\cal C}(\rho)$, $k\in\{0,1,\cdots,n \}$.
		\item [(b)] ${\mu _{L + k}}(h) = {\mu _{L + k + 1}}(h)$ for $h \in  \cup _{i = 0}^k{{\cal G}_i}$, $k\in\{0,1,\cdots,n-1 \}$.
		\item [(c)] $\mu _L^k(h) = \mu _L^{k + 1}(h)$ for $h \in  \cup _{i = 0}^k{{\cal Z}_i}$, $k\in\{0,1,\cdots,n-1 \}$.		
	\end{itemize}	
	Note that (b) follows from (a); (c) was discussed when defining TA($L$). Statement (1) of the theorem is a special case of (a). 
	
	We next prove statement (2) using (a), (b), and (c). Note that ${\mu _L} = \mu _L^0(h)$. It follows that $\mu _L^1(h) = \mu _L^0(h) = {\mu _L}(h) = {\mu _{L + 1}}(h)$ for $h \in {{\cal G}_0} = {{\cal Z}_0}$; that is, ${\cal C}^1_{L}(\rho)$ and ${\cal C}_{L+1}(\rho)$ have the same optimal action for every $h\in {\cal Z}_0$. Then according to Lemma \ref{lem:CLn}, $\varphi _L^1(h) = {\varphi _{L + 1}}(h)$ for all $h\in {\cal H}^1_L$. This fact, together with \eqref{eq:thm13-1} (the case of $k=1$), implies that $\mu^1_L(h)$ is an optimal action for $h \in {{\cal Z}_0} \cup {{\cal Z}_1}$ in ${\cal C}(\rho)$. That is, ${\mu _{L + 1}}(h) = \mu _L^1(h)$  for $h \in {{\cal Z}_0} \cup {{\cal Z}_1}$. Repeating the above argument until the case of $k=n$ yields statement (2).
	
\end{proof}

\subsection*{C.3 Proof of Theorem \ref{thm:eps-optimal}}
\begin{proof}
	Let $\mu$ be a policy that takes optimal actions in all reachable states in $\cup _{k = 0}^n{{\cal G}_k}$. Define $\tau_1=(h_0,\mu(h_0))$ and $\tau_{i+1}=(\tau_i,\mu(\tau_i))$ for $i\ge 1$. That is, $\tau_i$ is the reachable descendant of $h_0$ at layer $i$ under policy $\mu$. Then for any $h_0\in \mathcal{G}_0$,
	\begin{align} \label{eq:14-1}
		0\le & \varphi(h_0) - \varphi^\mu(h_0) =  \beta (1-\rho)[\varphi(\tau_1) - \varphi^\mu(\tau_1)] + \notag \\
		&\beta \sum_{h'\in \mathcal{G}_0}D(h'|h,a^*)[\varphi(h') - \varphi^\mu(h')],
	\end{align}
	where $a^*$ is the optimal action for $h_0$. The equality holds because $h_0\in\mathcal{G}_0$ and thus $\mu(h_0)=a^*$. Define
	\begin{align*}
		\Delta_i = \max_{h \in {\cal G}_i} |\varphi(h) - \varphi^\mu(h)|,\quad i=0,1,2,\cdots
	\end{align*}
	Then \eqref{eq:14-1} implies
	\begin{align*}
		|\varphi(h_0) - \varphi^\mu(h_0)| \le \beta \rho\Delta_0 + \beta (1-\rho)|\varphi(\tau_1) - \varphi^\mu(\tau_1)|.
	\end{align*}
	Note that the above inequality holds for all $h_0\in \mathcal{G}_0$. Therefore,
	\begin{align} \label{eq:14-3}
		\Delta_0 \le \beta \rho\Delta_0 + \beta (1-\rho)|\varphi(\tau_1) - \varphi^\mu(\tau_1)|.
	\end{align}
	Applying a similar argument as in \eqref{eq:14-1} yields
	\begin{align} \label{eq:14-4}
		\Delta_0 &\le \beta \rho\Delta_0 + \beta \rho \beta (1-\rho)\Delta_0 + \beta^2 (1-\rho)^2|\varphi(\tau_2) - \varphi^\mu(\tau_2)| \notag\\
		& \le \beta\rho \sum_{k=0}^{n} \beta^k (1-\rho)^k \Delta_0 + \beta^{n+1}(1-\rho)^{n+1}\Delta_{n+1} \notag\\
		&=\beta \rho \frac{1-\beta^{n+1}(1-\rho)^{n+1}}{1-\beta(1-\rho)}\Delta_0 + \beta^{n+1}(1-\rho)^{n+1}\Delta_{n+1}.
	\end{align}	
	Since $K=\max_{h,a}|\bar{R}(h,a)|$, we have
	\begin{align*}
		-\frac{K}{1-\beta} \le \varphi^u(h)\le \varphi(h)\le \frac{K}{1-\beta},\quad \forall h.
	\end{align*}
	It follows that $\Delta_{n+1} \le 2K/(1-\beta)$. We thus can obtain from \eqref{eq:14-4} that
	\begin{align} \label{eq:14-5}
		\Delta_0\le \frac{\beta^{n+1}(1-\rho)^{n+1}[1-\beta(1-\rho)]}{1-\beta +\beta \rho\beta^{n+1}(1-\rho)^{n+1} } \frac{2K}{1-\beta}.
	\end{align}	
	Given $\mathcal{C}(\rho)$ with initial position state $h_0\in \mathcal{G}_0$, policy $\mu$ achieves $\varepsilon$-optimal value if $|\varphi(h_0) - \varphi^\mu(h_0)| \le \Delta_0\le \varepsilon$. This is guaranteed if the RHS of \eqref{eq:14-5} is smaller than or equal to $\varepsilon$, which implies that
	\begin{align}
		\beta^{n+1}(1-\rho)^{n+1} &\le \frac{\varepsilon(1-\beta)^2 }{2K[1-\beta(1-\rho)]-\varepsilon\beta\rho(1-\beta) } \notag \\
		& \le \frac{\varepsilon(1-\beta) }{2K-\varepsilon\beta\rho }.
	\end{align}
	Therefore, $|\varphi(h_0) - \varphi^\mu(h_0)|\le \varepsilon$ if
	\begin{align*}
		n\ge \frac{\log \varepsilon(1-\beta)/(2K-\varepsilon\beta\rho) }{\log \beta(1-\rho)} -1.
	\end{align*}
	We thus conclude that a policy achieves $\varepsilon$-optimal value for any initial state $h_0\in\mathcal{G}_0$ if it takes optimal actions for all reachable states in $\cup _{k = 0}^n{{\cal G}_k}$, with $n$ satisfying the above inequality.
	Hence, given a satisfied $n$, TA($L+n$) and TA($n,L$) policies achieve $\varepsilon$-optimal value if $L$ satisfies the inequality stated in Theorem \ref{thm:policies}. Note that $\delta (n,L)\le K/(1-\beta)$ for any $n$ and $L$. Hence the inequality of Theorem \ref{thm:policies} holds if
	\begin{align*}
		\frac{K}{1-\beta}\left( {\frac{{{\beta ^{L + 1}}{{(1 - \rho )}^{L + 1}}}}{{1 - \beta }} + {\beta ^L}{{(1 - \rho )}^L}} \right) \le \epsilon(n),
	\end{align*}
	which implies that
	\begin{align*}
		L\ge \frac{\log\epsilon(n)(1-\beta)^2/K(1-\beta \rho) }{\log \beta(1-\rho)}.
	\end{align*}
	This completes the proof.	
\end{proof}

\section*{Appendix D}

\subsection*{D.1 Proof of Theorem \ref{thm:NVI}}

\begin{proof}
	Consider an arbitrary definition of $\{\mathcal{X}_i \}$ that satisfies $\mathcal{X}_d = \mathcal{H}_L$ and $\mathcal{X}_i \subseteq \mathcal{X}_{i+1}$. For any bounded real-valued function $\psi:{\cal H}_L\to \mathbb{R} $, define the operator $F_k$ as
	\begin{align*}
		{F_k}\psi (h) =  \begin{cases}
			\mathop {\max }\limits_{a \in {\cal A}} \left\{ {\bar R(h,a) + \beta \sum\limits_{h' \in {\cal H}{_L}} {{D_L}\left( {h'|h,a} \right)\psi (h')} } \right\}, \\ 
			\qquad \ {\text{  if }}h \in {{\cal X}_k}\\
			\psi (h),{\text{  otherwise}}
		\end{cases} 
	\end{align*}
	where $k\in \{1,2,\cdots,d \}$. That is, $F_k$ is the operator corresponding to the $(d-k+1)$-th inner iteration within an outer iteration. It is well-known that $F_d$ is a contraction mapping. We thus have, for any $h\in \mathcal{H}_L$,
	\begin{align} \label{eq:FL-contract}
		\left| {{F_k}{\psi ^n}(h) - {\varphi _L}(h)} \right|&=\left| {{F_d}{\psi ^n}(h) - {F_d}{\varphi _L}(h)} \right| \notag \\
		& \le \beta \max\limits_{h' \in {{\cal H}_L}}  \left| {{\psi ^n}(h') - {\varphi _L}(h')} \right|.
	\end{align}
	Using a similar argument as that establishes \eqref{eq:FL-contract}, we have
	\begin{align*}
		\begin{cases}
			\left| {{F_k}{\psi ^n}(h) - {\varphi _L}(h)} \right| \le \beta \max\limits_{h' \in {{\cal H}_L}} \left| {{\psi ^n}(h') - {\varphi _L}(h')} \right|, \forall  h\in \mathcal{X}_k \\
			\left| {{F_k}{\psi ^n}(h) - {\varphi _L}(h)} \right|\le \max\limits_{h' \in {{\cal H}_L}}\left| {{\psi ^n}(h) - {\varphi _L}(h)} \right|, \forall h\notin \mathcal{X}_k.
		\end{cases} 	
	\end{align*}
	for $1\le k <d$.		
	Let $m=kd+1$ for an arbitrary $k$. Then
	\begin{align}
		\left|\psi^{m+d}(h)-\varphi_L(h) \right|&=\left|F_dF_1\cdots F_{d-2}F_{d-1}\psi^{m}(h)-\varphi_L(h) \right| \nonumber\\
		&\le \beta \max\limits_{h' \in {{\cal H}_L}} \left| {{\psi ^m}(h') - {\varphi _L}(h')} \right|,
	\end{align}
	Note that the above inequality holds for any $h\in \mathcal{H}_L$. Hence for any $k$, 
	\begin{align*}
		\left|\left|\psi^{(k+1)d+1}(h)-\varphi_L(h) \right|\right|_\infty \le \beta  \left|\left| {{\psi ^{kd+1}}(h') - {\varphi _L}(h')} \right|\right|_\infty.
	\end{align*}
	The convergence of NVI thus follows immediately. 
	
	For statement 2, consider an arbitrary $k\ge 0$ and define
	\begin{align*}
		\Delta _i^n = \mathop {\max }\limits_{h \in {{\cal G}_i}} \left| {{\psi^{kd+n}}(h) - {\psi^{kd+n - 1}}(h)} \right|,\quad {\Delta ^n} = \mathop {\max }\limits_i \Delta _i^n.
	\end{align*}
	Note that the values of states in $\mathcal{H}_L-\mathcal{H}_1$ are updated every $d$ iterations. Hence $\Delta^n_i=0$ if $1\le n <d$ and $i\ge 2$.
	For any $h\in \mathcal{G}_0\subset \mathcal{X}_1$, we have
	\begin{align*}
		\left| {{\psi^{j+1}}(h) - {\psi^{j}}(h)} \right| \le \beta \rho \max_{h \in {\cal G}_0}\left| {{\psi^{j+1}}(h) - {\psi^{j}}(h)} \right|& \\
		+ \beta(1-\rho)\max_{h \in {\cal G}_1}\left| {{\psi^{j+1}}(h) - {\psi^{j}}(h)} \right|&,
	\end{align*}
	where $a$ is the greedy action w.r.t. $\psi^j$. It follows that
	\begin{align} \label{eq:delta_0}
		\Delta^{n+1}_0 \le \beta \rho \Delta^n_0 + \beta (1-\rho) \Delta^n_1 \le \beta\Delta^n.
	\end{align}
	Using a similar argument, it is easy to verified that $\Delta^{n+1}\le \Delta^n$ for all $1\le n <d$. Then we can derive from \eqref{eq:delta_0} that
	\begin{align} 
		\Delta^{d}_0 &\le \beta \rho \Delta^{d-1}_0 + \beta (1-\rho) \Delta^1 \notag\\
		&\le 
		{\beta ^{d-1}}{\rho ^{d-1}}{\Delta ^1} + \beta (1 - \rho )\sum\limits_{k = 0}^{d - 2} {{\beta ^k}{\rho ^k}{\Delta ^1}} \nonumber\\
		&= \left[ {{\beta ^{d-1}}{\rho ^{d-1}} + \beta (1 - \rho )\frac{{1 - {\beta ^{d-1}}{\rho ^{d-1}}}}{{1 - \beta \rho }}} \right]{\Delta ^1}.
	\end{align}
	At time step $kd+d+1$, updates are carried out over $\mathcal{X}_d=\mathcal{H}_L$. Therefore, for any $h\in \mathcal{H}_L$,
	\begin{align*}
		&\left| {{\psi^{kd+d+1 }}(h) - {\psi^{kd+d}}(h)} \right| \le \beta \rho \Delta _0^{d } + \beta (1 - \rho ){\Delta ^{d }}\\
		\le& \left[ {{\beta ^{d }}{\rho ^{d }} + \beta (1 - \rho )\frac{{\beta \rho  - {\beta ^{d }}{\rho ^{d }}}}{{1 - \beta \rho }} + \beta (1 - \rho )} \right]{\Delta ^1}\\
		=& \left[ {\frac{{\beta (1 - \rho )}}{{1 - \beta \rho }} + \frac{{{\beta ^{d}}{\rho ^{d}}(1 - \beta )}}{{1 - \beta \rho }}} \right]{\Delta ^1}.
	\end{align*}
	The desired result follows immediately.
\end{proof}

\subsection*{D.2 Proof of Theorem \ref{thm:complexity}}
\begin{proof}
	Given $\{\mathcal{X}_i \}$ defined by \eqref{eq:nestedset}, $|\mathcal{X}_1|=|\mathcal{S}|(1+|\mathcal{A}|)$ and $|\mathcal{X}_d|=|\mathcal{H}_L|$ is given by \eqref{eq:HL}.
	Hence each outer iteration consists of $O\left(d|\mathcal{X}_1|^2|\mathcal{A}|+|\mathcal{X}_d|^2|\mathcal{A}| \right)$ operations and the running time for each outer iteration is
	\begin{align} \label{eq:bigO}
		O\left( |\mathcal{S}|^2\left(d|\mathcal{A}|^3+|\mathcal{A}|^{2L+1}\right) \right).
	\end{align} 
	
	Next, we analyze the number of outer iterations required for convergence to an $\varepsilon$-optimal policy. The basic idea is similar to that outlined by \cite{littman1995complexity}. To ensure comprehensive coverage, we present a complete proof here.
	
	Let $K=\max_{h,a}|\bar{R}(h,a)|$. Then the value function of any policy is bounded in max norm by $K/(1-\beta)\triangleq K_v$. Let $n$ be the counter of outer iterations, then according to statement (2) of Theorem \ref{thm:NVI}, 
	\begin{align*}
		\left|\left| {{\psi ^{nd + 1}} - {\psi^{nd}}} \right|\right|_\infty \le \left[ {\frac{{\beta (1 - \rho )}}{{1 - \beta \rho }} + \frac{{{\beta ^{d}}{\rho ^{d}}(1 - \beta )}}{{1 - \beta \rho }}} \right]^n 2K_v .
	\end{align*}
	According to \cite{williams1993tight}, the greedy policy in terms of $\psi^{nd}$ is $\varepsilon$-optimal if $||{{\psi ^{nd + 1}} - {\psi^{nd}}}||_\infty \le \varepsilon (1-\beta)/2\beta$. This condition is guaranteed if
	\begin{align}
		\left[ {\frac{{\beta (1 - \rho )}}{{1 - \beta \rho }} + \frac{{{\beta ^{d}}{\rho ^{d}}(1 - \beta )}}{{1 - \beta \rho }}} \right]^n 2K_v \le \frac{\varepsilon (1-\beta)}{2\beta}.
	\end{align}
	It follows that the greedy policy in terms of $\psi^{nd}$ is $\varepsilon$-optimal for any $n$ satisfying
	\begin{align*}
		n \ge n^* \triangleq \frac{{\log {{{\varepsilon(1 - \beta )}}}/{{2K_v\beta }}}}{{\log \left[ {\frac{{\beta (1 - \rho )}}{{1 - \beta \rho }} + \frac{{{\beta ^{d}}{\rho ^{d}}(1 - \beta )}}{{1 - \beta \rho }}} \right]}} . 
	\end{align*}
	Using the fact that $\log x < x-1<0$ for $x\in (0,1)$, we can easily verify that
	\begin{align} \label{eq:n*}
		n^*\le \frac{\log 1/\varepsilon(1-\beta) + \log 2K_v }{1-\beta} \frac{1-\beta\rho}{1-\beta^{d}\rho^{d}}.
	\end{align}
	Combining \eqref{eq:bigO} and \eqref{eq:n*} yields the desired result. Note that we omit $\log 2K_v$ in the final expression because it is a constant.
\end{proof}

\section*{Appendix E}
This Appendix provides an introduction to the experimental MDP in Section \ref{sec:exp} (Table \ref{tab:hota} and Fig. \ref{fig:exp}).

As shown in Fig.~\ref{fig:boat}, consider an unmanned boat tasked with remaining within a designated target water area to perform monitoring activities. The boat’s movement is prone to failure with a certain probability due to the influence of water waves.
The boat is permitted to navigate within a predefined operational area, and leaving this area results in the termination of the task. Its position is tracked by a remote sensor, which transmits the location information to the boat at each time step. However, the long distance between the sensor and the boat introduces a communication failure probability of $1-\rho$. To simplify the problem, we discretize the water area into 9 segments, resulting in an MDP with 9 states and 4 actions.

The simplified MDP environment is shown in~\ref{fig:exmple}. In particular, we divide the permission area into eight distinct sections, each corresponding to a specific position. The obstacle area and the region outside the permission area are grouped together as a single section, representing the ninth position.  The boat has four actions at any position: move left, move right, move up, and move down. The aim is to control the boat to keep moving along a clockwise direction in the white region ($1 \to 2 \to  \cdots  \to 8 \to 1 \to  \cdots $). The task is terminated once the boat enters the gray region (including the gray region in the central part and the gray region in the marginal part).

\begin{figure}[!t]
	\centering
	\begin{subfigure}[b]{0.4\linewidth}
		\includegraphics[width=\linewidth]{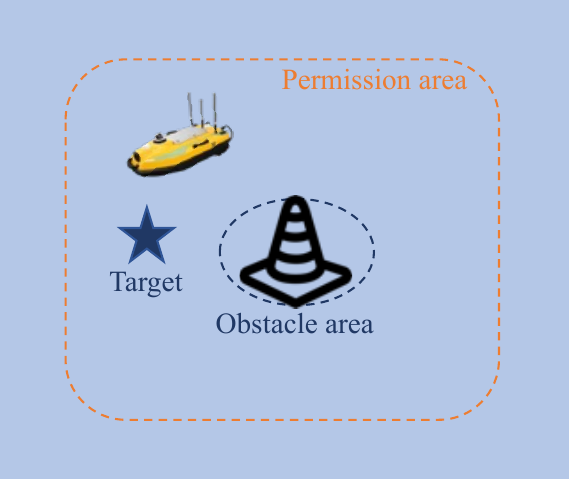}
		\caption{Illustration of practical environment.}
		\label{fig:boat}
	\end{subfigure}
	\begin{subfigure}[b]{0.37\linewidth}
		\includegraphics[width=\linewidth]{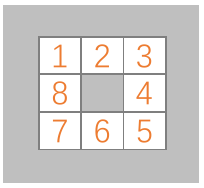}
		\caption{Simplified MDP environment.}
		\label{fig:exmple}
	\end{subfigure}
	
	\caption{Environment of the unmanned Boat MDP}
	\label{fig:env}
\end{figure}

Given a positive integer $n$, we use $[n]=\{1,2,\cdots,n\}$  to denote the set of integers between 1 and $n$. Mathematically, the MDP is formulated as follows:
\begin{itemize}
	\item State space ${\cal S} = [9]$, where states 1-8 correspond to positions 1-8 in the white region and state 9 corresponds to the gray region (terminated state).
	\item Action space ${\cal A} = [4]$: action 1, move left; action 2, move down; action 3, move right; action 4, move up.
	\item Transition kernel. For any state $s\in [8]$, denote by $m_s\in {\cal A}$ the clockwise action of state $s$ (that is, move left in states 1 and 2, move down in states 3 and 4, move right in states 5 and 6, move up in states 7 and 8) and $m^-_s\in {\cal A}$ the anti-clockwise action. Denote by $s_t$ and $a_t$ the state and action at time $t$, respectively. The transition kernel is defined as follows: (1) for $s_t = 9$, $\Pr({s_{t + 1}}=j|{s_t} = 9,{a_t}) = 1$ if $j=9$ and $\Pr({s_{t + 1}} = j|{s_t} = 9,{a_t}) = 0$ otherwise; (2) for $s_t\in [8]$,
	\begin{align*}
	\Pr({s_{t + 1}} = j|{s_t} = i,{a_t} = {m_i}) =  \begin{cases}
	0.5,&{\text{  if }}j = i \in [8]\\
	0.5,&{\text{  if }}i \in [7]{\text{ and }}j = i + 1\\
	0.5,&{\text{  if }}i = 8{\text{ and }}j = 1\\
	0,&{\text{   otherwise}}
	\end{cases} 
	\end{align*}
	\begin{align*}
	\Pr({s_{t + 1}} = j|{s_t} = i,{a_t} = {m^-_i}) =  \begin{cases}
	0.5,&{\text{  if }}j = i \in [8]\\
	0.5,&{\text{  if }}2\le i\le 8{\text{ and }}j = i - 1\\
	0.5,&{\text{  if }}i = 1{\text{ and }}j = 8\\
	0,&{\text{   otherwise}}
	\end{cases} 
	\end{align*}
	\begin{align*}
	\Pr({s_{t + 1}} = j|{s_t} = i,{a_t} ) =  \begin{cases}
	1,&{\text{  if }}j =  9 \text{ and } a_t\notin \{m_i, m^-_i\}\\
	0,&{\text{   otherwise}}
	\end{cases} 
	\end{align*}
	\item Reward function: $r(s,a)=20$ if $s\in [8],a=m_s$ and $r(s,a)=0$ otherwise.
	\item Discount factor: $\beta = 0.95$.
\end{itemize}

\vskip 0.2in
\bibliography{sample}

\end{document}